\documentclass[11pt,reqno]{amsart}
\usepackage{amsaddr}
\usepackage[margin=1in]{geometry}
\usepackage{bbm}

\usepackage{tabularx, makecell, booktabs, adjustbox}
\usepackage{rotating}

\newcommand{\Om}{\Omega}
\usepackage{graphicx}
\usepackage{hyperref}
\usepackage{url}
\usepackage{color}

\usepackage{amsmath,amsfonts,bm}

\def\1{\bm{1}}

\DeclareMathAlphabet{\mathsfit}{\encodingdefault}{\sfdefault}{m}{sl}
\SetMathAlphabet{\mathsfit}{bold}{\encodingdefault}{\sfdefault}{bx}{n}

\newcommand{\sigmoid}{\sigma}

\DeclareMathOperator*{\argmax}{arg\,max}

\usepackage{caption}
\usepackage{algpseudocode,algorithm}
\usepackage{booktabs}
\usepackage{multirow}
\usepackage{enumitem}

\usepackage{subcaption}
\usepackage{hyperref}
\usepackage{url}
\usepackage{amsmath,amsthm}
\usepackage{amssymb,bm}
\usepackage{bbm}
\usepackage{makecell}
\newtheorem{theorem}{Theorem}

\newtheorem{lemma}{Lemma}
\newtheorem{proposition}{Proposition}
\newtheorem{remark}{Remark}

\usepackage{pgfplots}
\pgfplotsset{compat=1.10}
\usetikzlibrary{
        patterns,
        pgfplots.fillbetween,
    }

       \tikzset{
        hatch distance/.store in=\hatchdistance,
        hatch distance=10pt,
        hatch thickness/.store in=\hatchthickness,
        hatch thickness=2pt,
        hatch color/.store in=\hatchcolor,      
        hatch color=black,                      
    }
    \pgfdeclarepatternformonly[%
        \hatchdistance,%
        \hatchthickness,%
        \hatchcolor
            ]{flexible hatch}
    {\pgfqpoint{0pt}{0pt}}
    {\pgfqpoint{\hatchdistance}{\hatchdistance}}
    {\pgfpoint{\hatchdistance-1pt}{\hatchdistance-1pt}}%
    {
        \pgfsetlinewidth{\hatchthickness}
        \pgfpathmoveto{\pgfqpoint{0pt}{0pt}}
        \pgfpathlineto{\pgfqpoint{\hatchdistance}{\hatchdistance}}
        \pgfsetstrokecolor{\hatchcolor}
        \pgfusepath{stroke}
    }

\def\argmax{\operatornamewithlimits{arg\,max}}

\usepackage[makeroom]{cancel}
\definecolor{darkgreen}{rgb}{0.0, 0.5, 0.0}
\definecolor{darkred}{rgb}{0.5, 0.0, 0.0}

\definecolor{dorkgreen}{rgb}{0.0, 0.8, 0.0}
\hypersetup{
     colorlinks = true,
     linkcolor = red,
     anchorcolor = blue,
     citecolor = teal,
     filecolor = blue,
     urlcolor = black
     }

\usepackage{wrapfig}  

\newtheorem{assumption}{Assumption}

\DeclareMathOperator{\logit}{logit}

\usepackage{enumitem}
\usepackage{xcolor}
\usepackage{tcolorbox}

\definecolor{TakeawayAccent}{HTML}{1F77B4} 

\tcbset{
  takeaway/.style={
    colback=TakeawayAccent!5,          
    colframe=TakeawayAccent!55!black,  
    boxrule=0.7pt,
    arc=2mm,
    left=9pt, right=9pt, top=7pt, bottom=7pt,
  },
  takeawaytitle/.style={
    colback=TakeawayAccent!10,
    colframe=TakeawayAccent!60!black,
    boxrule=0.7pt,
    sharp corners
  }
}

\newtcolorbox{takeaway}[1][]{
  takeaway,
  title={#1},
  fonttitle=\bfseries,
}

\newenvironment{lessons}{%
  \begin{itemize}[leftmargin=*, topsep=2pt, itemsep=2pt]
}{%
  \end{itemize}
}

\newtcolorbox{assbox}{
  colback=black!5!white,
  colframe=black!75!black,
  width=\textwidth,
}
\newtcolorbox{thmbox}{colback=blue!5!white,colframe=black!75!black}
\usepackage{fancyref}
\usepackage{natbib}

\hypersetup{
     colorlinks = true,
     linkcolor = red,
     anchorcolor = blue,
     citecolor = teal,
     filecolor = blue,
     urlcolor = black
     }

\makeatletter
\providecommand\paragraph{\@startsection{paragraph}{4}{\z@}%
  {3.25ex \@plus1ex \@minus.2ex}{-1em}%
  {\normalfont\normalsize\bfseries}}
\providecommand\subparagraph{\@startsection{subparagraph}{5}{\z@}%
  {3.25ex \@plus1ex \@minus.2ex}{-1em}%
  {\normalfont\normalsize\bfseries}}
\makeatother

\title[Reinforcement Learning with Verifiable Rewards]{Reinforcement Learning with Verifiable Rewards:\\ GRPO's Effective Loss, Dynamics, and Success Amplification}
\author[Y. Mroueh]{Youssef Mroueh}
\address[Y.Mroueh]{IBM Research}

\definecolor{darkgreen}{rgb}{0.0, 0.5, 0.0}
\definecolor{darkred}{rgb}{0.5, 0.0, 0.0}

\definecolor{dorkgreen}{rgb}{0.0, 0.8, 0.0}

\begin{document}

\maketitle

\begin{abstract}
Group Relative Policy Optimization (GRPO) was introduced in \citep{shao2024deepseekmath} and used to train DeepSeek\textendash R1 \citep{guo2025deepseekR1} for promoting reasoning in LLMs under verifiable (binary) rewards. We show that the mean{+}variance calibration of these rewards induces a contrastive loss in which the contrastive samples are synthetic data drawn from the previous policy. While GRPO was originally paired with clipping to keep updates near the old policy, we analyze variants that differ in reward normalization (mean-only vs.\ mean{+}variance) and in how they regularize updates using KL divergence: either penalizing divergence from the previous model (\emph{mirror}), penalizing divergence from a fixed reference model $\pi_{\mathrm{ref}}$, or combining both forms of regularization. For each, the optimal policy $\pi_n$ admits an explicit form in terms of the binary reward and the first and second  order statistics of the reward under $\pi_{n-1}$, as well as the policies $\pi_{n-1}$ and $\pi_{\mathrm{ref}}$. Iterating results in a sequence  $\{\pi_n\}$ whose \emph{probability of success (PoS)} obeys a simple recurrence that converges to a fixed point determined by the reference PoS and the regularization strength. We further show that this fixed point exceeds the reference, demonstrating that GRPO amplifies the policy's probability of success.
\end{abstract}



\section{Introduction}

In Reinforcement Learning (RL), a policy is learned by maximizing a reward that encodes constraints or an objective we want the policy to conform to or achieve. Policy gradient methods and actor-critic methods \citep{sutton1998reinforcement}, enable RL-based training of parametric policies, including Large Language Models (LLMs), particularly when dealing with non-differentiable rewards. Unlike supervised learning or preference optimization, which require labeled training data, reinforcement learning generates \emph{synthetic data} sampled online from the learned policy as training progresses.

Proximal Policy Optimization (PPO), introduced in \citep{schulman2017proximal}, is a widely used algorithm that facilitates such training. PPO relies on importance sampling from the model’s previous (``old'') policy while ensuring that updates remain within a certain proximity to the old policy. Policy gradient methods are known for their high variance, and PPO mitigates this by learning a critic that reduces the variance of gradient estimates. The critic normalizes the reward, and PPO's advantage function—defined as the difference between the reward and the critic’s evaluation—drives the optimization process.

Group Relative Policy Optimization (GRPO) was recently introduced in DeepSeekMath \citep{shao2024deepseekmath}. GRPO closely follows PPO’s optimization framework but differs in how the advantage is computed. Specifically, GRPO estimates the advantage using Monte Carlo rollouts rather than a learned critic. Additionally, GRPO applies whitening to the advantage function, meaning it standardizes the reward’s mean and variance. These statistics are estimated from a ``group'' of Monte Carlo rollouts corresponding to  samples from the LLM policy conditioned on a single input or query to the policy. Whitening the advantage function has been recognized in many PPO implementations as an important ingredient for training stability \citep{Engstrom2020Implementation,huang2024the}. 

GRPO therefore eliminates the need for training a separate critic network alongside the LLM policy, instead leveraging efficient sampling from the LLM’s policy. This is made feasible by optimized model serving through VLLM \citep{kwon2023efficient}. GRPO has been employed in the DeepSeek model series, including DeepSeek-v3 \citep{liu2024deepseekv3} and DeepSeek-R1 \citep{guo2025deepseekR1}. DeepSeek-R1 unlocked reasoning capabilities in open-source models, and its success can be attributed to several factors and innovations, among them:  (1) A strong pre-trained model (DeepSeek-v3), (2) The reasoning chain of thoughts \texttt{<think>...<think> <answer>...<answer>}
and (3) The use of verifiable binary rewards with GRPO to fine-tune the models on reasoning and math tasks.

We focus in this paper on Reinforcement Learning with Verifiable Rewards (RLVR) using GRPO, as recently termed by \citet{lambert2024t}. Verifiable rewards for RL with LLMs typically include (i) correctness checks via string matching to a gold answer when available or via an LLM-as-judge otherwise \citep{guo2025deepseekR1,openr1,deepscaler2025,guan2025deliberativealignmentreasoningenables}. 
Additionally, (ii) execution-based pass/fail in code generation (interpreters and unit tests) and (iii) simple binary checks for formatting/refusals provide scalable 0/1 signals for training \citep{openr1,guo2025deepseekR1,lambert2024t}.
Verifiable reward balance simplicity and  bias and are thought to be less prone to reward hacking  than reward models learned from preference data. 
We note that a recent paper \citep{vojnovic2025alignmentobjectivegrpo} studies GRPO with a focus on the policy obtained using an approximation of the $\mathsf{KL}$  divergence used in practical implementations.

The original GRPO's practical recipe  \citep{shao2024deepseekmath} combines PPO-style clipping with an explicit KL regularizer to a frozen reference model.  On the other hand, mirror-descent style updates that regularize to the \emph{previous} iterate (rather than a fixed reference) have been studied under the Mirror Descent Policy Optimization (MDPO) framework, which interprets each step as approximately solving a trust-region problem via a Bregman (KL) proximity term to $\pi_{n-1}$ (see for example \citep{pmlr-v37-schulman15,tomar2021mirrordescentpolicyoptimization,gunter2024appleintelligencefoundationlanguage}). “Dr.\ GRPO” \citep{liu2025understanding} is a  variant that removes variance normalization (i.e., uses mean-only normalization of group rewards), simplifying the scaling while keeping the same overall training loop. Finally, recent large-scale systems such as DAPO \citep{yu2025dapoopensourcellmreinforcement}  report strong results when removing the reference-model KL entirely (i.e., training reference-free), alongside additional engineering choices such as decoupled clipping and dynamic sampling.

Our main contributions are:
\begin{enumerate}
\item \textbf{Contrastive Loss (Sec.~\ref{sec:CL}).}
We show that GRPO with calibrated verifiable rewards is equivalent to an \emph{adaptive, weighted contrastive loss} evaluated on samples from the previous policy.

\item \textbf{Policy Recursions.}
Leveraging this equivalence, we derive, for multiple GRPO variants, a closed-form recursion for the optimal policy as a function of $\pi_{\mathrm{ref}}$, $\pi_{n-1}$, and the previous policy’s probability of success (PoS) $p_{n-1}$.
Section~\ref{sec:Recursion} analyzes GRPO (no clipping) with a KL penalty to the reference; Section~\ref{Sec:MirrorGRPO} studies \emph{Mirror GRPO} with a KL penalty to the previous iterate only; Appendix~\ref{sec:two-kl} covers the mixed (two-KL) case i.e mixed KL penalties to reference and previous iteration; and Section~\ref{sec:meanonly} treats the mean-only normalization.

\item \textbf{PoS Dynamics \& Fixed-Point Amplification.}
We prove that the induced PoS sequence $(p_n)$ satisfies a recursion admitting a fixed point $p^*$ and, under mild assumptions, $p_n \to p^*$ with $p^* \ge p_{\mathrm{ref}}$, establishing \emph{success amplification} for GRPO. The stepwise monotonicity of $(p_n)$ depends on the specific variant. The dynamic of the PoS is verified empirically in Appendix \ref{Sec:Exp}. Code is provided in supplementary material. 
\end{enumerate}

\section{ GRPO With verifiable Rewards as an Adaptive Weighted Contrastive Loss } \label{sec:CL}

Let $\rho_{\mathcal{Q}}$ be a distribution of prompts or questions, and let $r$ be a reward function that evaluates the output $o \in \mathcal{O}$ of a  policy. As discussed in the introduction, we restrict our analysis to verifiable rewards, meaning binary rewards, $r: \mathcal{Q}\times \mathcal{O} \to \{0,1\}$.   Given a prompt $q \sim \rho_{Q}$, let $\pi_{\theta} (o|q)$ be the policy of an LLM, where $o$ represents the sequence outcome and  $\theta \in \Theta$ the parameters of the model. $\pi_{\theta_{\text{old}}}$ denotes the ``old'' policy or the policy from a previous iteration. $\pi_{\mathrm{ref}}$ corresponds to the reference policy, and $\mathsf{KL}$ is the  Kullback–Leibler divergence :
\[ \mathsf{KL} ( \pi || \pi_{\mathrm{ref}}) = \mathbb{E}_{q\sim\rho_{\mathcal{Q}} }\mathbb{E}_{o\sim \pi (.|q)} \log\left(\frac{ \pi (o|q)}{ \pi_{\text{ref}}(o|q)}\right) \]
We note the mean and variance of the reward under a policy $\nu$ as follows: 
\[ \mu_{\nu}(q)=\mathbb{E}_{o' \sim \nu (.|q)} r(q,o') \quad \sigma^2_{\nu}(q)= \mathsf{Var}_ {o' \sim \nu(.|q)} r(q,o').  \]

\noindent For a regularization parameter $\beta>0$, we  start by recalling  GRPO's optimization problem \citep{shao2024deepseekmath}: 
\begin{equation}
\max_{\theta} \mathbb{E}_{q\sim \rho_{\mathcal{Q}}} \mathbb{E}_{o\sim \pi_{\theta_{\text{old}}} (. |q) } f_{\epsilon}\left( \frac{\pi_{\theta} (o| q )}{\pi_{\theta_{\text{old}}}(o| q) }, A_{\pi_{\theta_{\text{old}}}}(q,o)  \right) - \beta \mathsf{KL} (\pi_{\theta} || \pi_{\mathrm{ref}}) 
\label{eq:GRPO}
\tag{GRPO-Clip}
\end{equation}
where the ``advantage" for an outcome $o$, $A(q,o)$ is given by the whitened reward:
\begin{equation}
A_{\pi_{\theta_{\text{old}}}}(q,o) = \frac{  r (q,o)-  \mu_{\pi_{\theta_{\text{old}}}}(q)  }{ \sigma_{\pi_{\theta_{\text{old}}}}(q)},
\label{eq:adv}
\end{equation}
and for $\epsilon \in [0,1]$, the clipping function
$f_{\epsilon}$  is given by  $ f_{\epsilon}(x,y)= \min(x y, \text{clip}(x,1-\epsilon, 1+ \epsilon) y ).$

\noindent We see that GRPO optimizes the whitened reward (referred to as advantage, $A(q,o)$, in \citep{shao2024deepseekmath}) using importance sampling from the ``old'' policy while maintaining the optimized policy close to  $\pi_{\mathrm{ref}}$ as measured by the $\mathsf{KL}$ divergence. The clipping used in \eqref{eq:GRPO} ensures that the likelihood ratio between the policy and the old policy is maintained within  $[1-\epsilon,1+ \epsilon]$.

\subsection{Whitening the Rewards in GRPO As means of  Calibration }

Recall that our reward $r$ is a verifiable reward that evaluates correctness of a reasoning or code  execution, so $r(q,o) \in \{0,1\}$. 
We note the probability of success of the old  policy $\pi_{\text{old}}$:
\begin{equation}
p(q) =  p_{\theta_{\text{old}}} (q) = \mathbb{P}_{o \sim \pi_{\theta_{\text{old}}(.|q) }}  ( r(q,o)=1 ) 
\end{equation}
Hence, for a Bernoulli random variable, the mean and variance are::
\[ \mu_{\pi_{\theta_{\text{old}}}}(q)   = p(q) \text{ and } \sigma^2_{\pi_{\theta_{\text{old}}}}(q) = p(q)(1-p(q)) .\]
Let us assume in the following that \underline{ $0<p(q)<1$} so that $\sigma^2_{\pi_{\theta_{\text{old}}}}(q)>0$.   Replacing mean and variance in  the whitened reward  in \eqref{eq:adv} we obtain :

\begin{equation}
\begin{array}{cc}
\displaystyle
A_{\pi_{\theta_{\text{old}}}}(q,o) =
\begin{cases} 
\frac{1-p(q)}{\sqrt{p(q)(1-p(q))}} & \text{if } r(q,o) =1, \\
-\frac{p(q)}{\sqrt{p(q)(1-p(q))}} & \text{if } r(q,o)= 0.
\end{cases}
&
\text{i.e,}
\quad
\displaystyle
A_{\pi_{\theta_{\text{old}}}}(q,o) =
\begin{cases} 
\sqrt{\frac{1-p(q)}{p(q)}} & \text{if } r(q,o) =1, \\
-\sqrt{\frac{p(q)}{1-p(q)}} & \text{if } r(q,o)= 0.
\end{cases}
\end{array}
\label{eq:advbinary}
\end{equation}

\paragraph{Calibrated reward behavior.}
We see that the whitening or the normalization of the verifiable reward in GRPO, calibrates the reward with respect to the conditional distribution of the reward under $\pi_{\theta_{\text{old}}}(.|q)$ for every prompt $q$. This normalization results in a calibration of the reward that involves non linear functions of  the probability of the success (PoS) of the old policy $p(q)$. See  Figure \ref{fig:calibration} for an illustration.
For a correct answer $r(q,o)=1$, the calibrated reward is positive and \emph{decreases} with the PoS $p(q)$: rare successes (small $p(q)$) receive more credit than easy ones (large $p(q)$).
For an incorrect answer ($r(q,o)=0$), the calibrated reward is negative, and its absolute value is increasing with $p(q)$. Wrong outcomes are more penalized when success is likely (for high $p(q)$) and less penalized when success is rare (low $p(q)$).

\begin{figure}[t!]
    \centering  
    \includegraphics[width=0.5\textwidth]{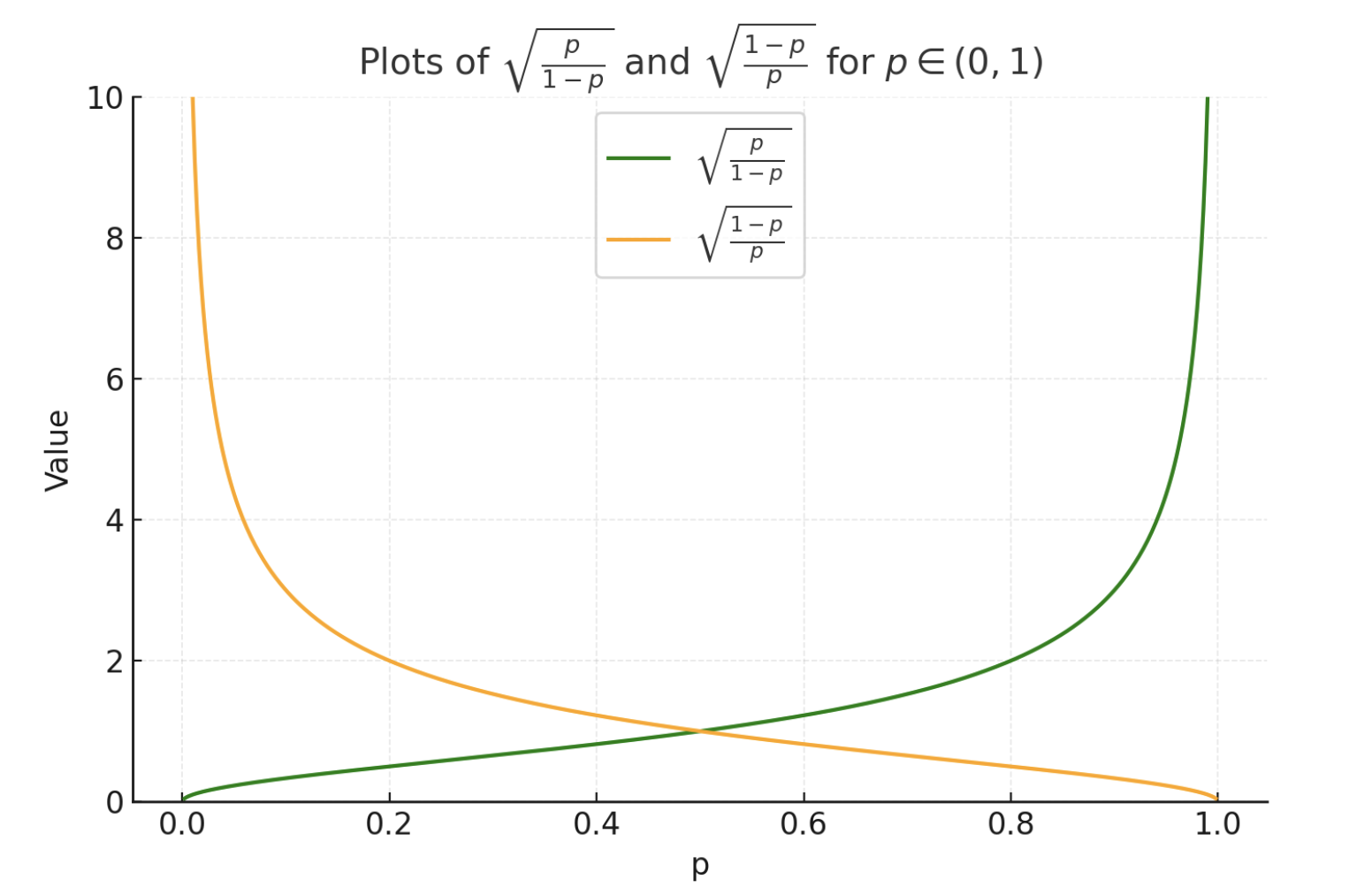}  
    \caption{Weighting of GRPO with the probability of success of the old policy.}  
    \label{fig:calibration}  
\end{figure}

\subsection{GRPO with verifiable Reward As a Weighted Contrastive Loss}

Replacing the calibrated reward $A_{\pi_{\theta_{\text{old}}}}(q,o)$ (\eqref{eq:advbinary}) for a verifiable reward in GRPO's optimization objective \eqref{eq:GRPO} we obtain:  
\begin{align*}
\mathbb{E}_{o\sim \pi_{\theta_{\text{old}}} (. |q) } f_{\epsilon}\left( \frac{\pi_{\theta} (o| q )}{\pi_{\theta_{\text{old}}}(o| q) }, A_{\pi_{\theta_{\text{old}}}}(q,o) \right) &= \sqrt{\frac{1-p(q)}{p(q)}} \, \mathbb{E}_{o\sim \pi_{\theta_{\text{old}}} (. |q),\, r(q,o)=1} \min\left(\frac{\pi_{\theta} (o| q )}{\pi_{\theta_{\text{old}}}(o| q)}, 1+\epsilon \right)\\
&- \sqrt{\frac{p(q)}{1-p(q)}} \, \mathbb{E}_{o\sim \pi_{\theta_{\text{old}}} (. |q),\, r(q,o)=0} \max\left(\frac{\pi_{\theta} (o| q )}{\pi_{\theta_{\text{old}}}(o| q)}, 1-\epsilon \right),
\end{align*}
where we used that for $x>0$ and $y>0$, $f_{\epsilon}(x,y)=x\min(y,1+\epsilon)$ and for $x>0,y<0$, $f_{\epsilon}(x,y)=x\max(y,1-\epsilon)$.

The overall cost is further obtained by taking expectation over $q$, noting $p(q)= p_{\theta_{\text{old}}} (q) $:
$$\textcolor{orange}{ \mathbb{E}_{q\sim \rho_{\mathcal{Q}}} \sqrt{\frac{1-p_{\theta_{\text{old}}} (q) }{p_{\theta_{\text{old}}} (q) }} \mathbb{E}_{o\sim \pi_{\theta_{\text{old}}} (. |q) } \min\left(\frac{\pi_{\theta} (o| q )}{\pi_{\theta_{\text{old}}}(o| q)},1+\epsilon \right)\mathbbm{1}_{r(q,o)=1}}$$
$$-\textcolor{darkgreen}{\mathbb{E}_{q\sim \rho_{\mathcal{Q}}} \sqrt{\frac{p_{\theta_{\text{old}}} (q) }{(1-p_{\theta_{\text{old}}} (q) )}} \mathbb{E}_{o\sim \pi_{\theta_{\text{old}}} (. |q) } \max \left(\frac{\pi_{\theta} (o| q )}{\pi_{\theta_{\text{old}}}(o| q)}, 1-\epsilon \right ) \mathbbm{1}_{r(q,o)=0} }- \beta \mathsf{KL} (\pi_{\theta} || \pi_{\mathrm{ref}}) $$
~\\
\noindent We see that GRPO is effectively a weighted contrastive loss that is weighted by a ratio depending on the probability of success of $\pi_{\theta_{\text{old}}} (.|q)$.
We see  from the weights plots  that if the success probability of the old policy is high  ($p >0.5$), the weighting for points with success is low since the old policy is already good, and for failing points the weight is high and hence they are more penalized.
On the other hand if the success probability of old policy is low ($p<0.5$), the weighting for points with success is high since we want to reinforce those successes, and for failing points these are still penalized but with a small weight.



\subsection{Stabilized GRPO } Note that in the previous sections we assumed that $0<p(q)<1$, so we ensure $\sigma^2_{\pi_{\theta_{\text{old}}}}(q)>0$. In the following, we alleviate this in the following by adding a smoothing factor $\varepsilon \in (0,1]$ in the advantage as follows:
 $$A_{\pi_{\theta_{\text{old}}}}(q,o)= \frac{r(q,o) -\mu_{\pi_{\theta_{\text{old}}}}(q)}{\sqrt{\sigma^2_{\pi_{\theta_{\text{old}}}}(q)+\varepsilon}}.$$
 This results with the following stabilized whitened reward: 
\begin{equation}
\label{eq:casesmeanvar}
A_{\pi_{\theta_{\text{old}}}}(q,o)=
\begin{cases}
+\omega^+_{\varepsilon}(p(q)), & r(q,o)=1,\\
-\omega^-_{\varepsilon}(p(q)), & r(q,o)=0,
\end{cases}
\quad
\omega^+_{\varepsilon}(p)=\frac{1-p}{\sqrt{p(1-p)+\varepsilon}},
\qquad
\omega^-_{\varepsilon}(p)=\frac{p}{\sqrt{p(1-p)+\varepsilon}},
\end{equation}
with smoothing $\varepsilon>0$. 

Replacing the stabilized advantage in \eqref{eq:GRPO},  we obtain the following contrastive optimization:
 \begin{align*}
&   \max_{\theta}  \mathbb{E}_{q\sim \rho_{\mathcal{Q}}} \Big \{  \omega^+_{\varepsilon} (p_{\theta_{\text{old}}}(q) ) \mathbb{E}_{o\sim \pi_{\theta_{\text{old}}} (. |q) }  \min\left(\frac{\pi_{\theta} (o| q )}{\pi_{\theta_{\text{old}}}(o| q)},1+\epsilon \right)\mathbbm{1}_{r(q,o)=1} \\
& ~~~~~~~~~~~~~~~~~~~~~~ - \omega^-_{\varepsilon} (p_{\theta_{\text{old}}}(q) ) \mathbb{E}_{o\sim \pi_{\theta_{\text{old}}} (. |q) } \max \left(\frac{\pi_{\theta} (o| q )}{\pi_{\theta_{\text{old}}}(o| q)}, 1-\epsilon \right ) \mathbbm{1}_{r(q,o)=0} \Big \}  - \beta \mathrm{KL} (\pi_{\theta} || \pi_{\mathrm{ref}}) 
\end{align*}

 \paragraph{Stabilized GRPO with No Clipping}
 
 \noindent To simplify Equation \ref{eq:GRPO}, let us consider this objective without the clipping ($\epsilon \to + \infty$); we obtain:
\begin{equation}
\max_{\theta} \mathbb{E}_{q\sim \rho_{\mathcal{Q}}} \mathbb{E}_{o\sim \pi_{\theta_{\text{old}}} (. |q) }\frac{\pi_{\theta} (o| q )}{\pi_{\theta_{\text{old}}}(o| q) } A_{\pi_{\theta_{\text{old}}}}(q,o) - \beta \mathsf{KL} (\pi_{\theta} || \pi_{\mathrm{ref}}) 
\label{eq:GRPO-noClip}
\tag{GRPO}
\end{equation}
~~\\
 Taking the clipping parameter $\epsilon \to \infty$ we obtain GRPO with no clipping equivalent contrastive optimization as follows: 
 \begin{align}
&  \max_{\theta}  \mathbb{E}_{q\sim \rho_{\mathcal{Q}}}  \Big \{  \omega^+_{\varepsilon} (p_{\theta_{\text{old}}}(q) )\mathbb{E}_{o\sim \pi_{\theta_{\text{old}}} (. |q) } \frac{\pi_{\theta} (o| q )}{\pi_{\theta_{\text{old}}}(o| q)}\mathbbm{1}_{r(q,o)=1}\nonumber\\
& -  \omega^-_{\varepsilon} (p_{\theta_{\text{old}}}(q) ) \mathbb{E}_{o\sim \pi_{\theta_{\text{old}}} (. |q) } \frac{\pi_{\theta} (o| q )}{\pi_{\theta_{\text{old}}}(o| q)} \mathbbm{1}_{r(q,o)=0}  \Big \}- \beta \mathrm{KL} (\pi_{\theta} || \pi_{\mathrm{ref}}) 
\label{eq:CLGRPOnoClip}
\tag{GRPO-No-Clip}
\end{align}

\noindent which is equivalent to the following problem: 

\begin{equation}
 \max_{\theta}  \mathbb{E}_{q\sim \rho_{\mathcal{Q}}} \Big \{  \omega^+_{\varepsilon}(p_{\theta_{\text{old}}}(q))  \mathbb{E}_{o\sim \pi_{\theta(. |q) } }\mathbbm{1}_{r(q,o)=1}
- \omega^-_{\varepsilon}(p_{\theta_{\text{old}}}(q)) \mathbb{E}_{o\sim \pi_{\theta} (. |q) } \mathbbm{1}_{r(q,o)=0} \Big\} - \beta \mathrm{KL} (\pi_{\theta} || \pi_{\mathrm{ref}}), 
\label{eq:GRPOnoClip}
\end{equation}

We will focus first  on this non-clipped  version.

\subsection{GRPO Iterations} Algorithm \ref{alg:iter-grpo} in Appendix \ref{sec:alg}  summarizes GRPO iterations (Stabilized and no clipping). We see that  GRPO iterations can be written as a sequence  of optimization resulting in policies we  denote by $\pi_{\theta_n}$ the policy at iteration $n$.  We see that GRPO iterations  can be written for $n\geq 1$:

\begin{equation}
\theta_{n} = \argmax_{\theta}  \mathbb{E}_{q\sim \rho_{\mathcal{Q}}} \Big \{ \omega^+_{\varepsilon}(p_{\theta_{n-1}}(q))  \mathbb{E}_{o\sim \pi_{\theta(. |q) } }\mathbbm{1}_{r(q,o)=1}
- \omega^-_{\varepsilon}(p_{\theta_{n-1}}(q)) \mathbb{E}_{o\sim \pi_{\theta} (. |q) } \mathbbm{1}_{r(q,o)=0} \Big \}  - \beta \mathrm{KL} (\pi_{\theta} || \pi_{\mathrm{ref}}), 
\label{eq:grpoparametric}
\end{equation} 
Note that in Algorithm \ref{alg:iter-grpo}, expectations are estimated using importance sampling from $\pi_{\theta_{n-1}}$, and each maximization problem is solved via gradient for $\mu$ steps. PoS are estimated using a group size $G$, i.e $G$ Monte-Carlo rollouts from $\pi_{\theta_{\text{old}}}(.|q)$.

In the following we will replace the maximization on the parameter space of the policy by maximizing over the space of  policies (i.e optimization on the probability space) in order to analyze the dynamics of GRPO iterations as follows,  for $n\geq 1$:
\begin{equation}
\pi_{n} = \argmax_{\pi}  \mathbb{E}_{q\sim \rho_{\mathcal{Q}}} \Big \{ \omega^+_{\varepsilon}\left(p_{n-1}(q)\right)  \mathbb{E}_{o\sim \pi(. |q) } \mathbbm{1}_{r(q,o)=1}
- \omega^-_{\varepsilon}(p_{n-1}(q)) \mathbb{E}_{o\sim \pi (. |q) } \mathbbm{1}_{r(q,o)=0} \Big \}  - \beta \mathrm{KL} (\pi|| \pi_{\mathrm{ref}}), 
\label{eq:GRPO_iterations}
\tag{GRPO Iterations}
\end{equation}
where $p_{n-1}(q)$ is the probability of success of the policy $\pi_{n-1}(\cdot|q)$:
\begin{equation}
p_{n-1}(q)= \mathbb{E}_{o \sim \pi_{n-1} ( .| q )} \mathbbm{1}_{r(q,o)=1} 
\label{eq:psuccess}
\end{equation}
and the weights $\omega^+_\varepsilon$ and $\omega^-_{\varepsilon}$ are given in  \eqref{eq:casesmeanvar}. We assume  all throughout the paper that $\pi_0= \pi_{\mathrm{ref}}.$
Note that moving the optimization from a parametric space to the probability space can be seen as assuming that the hypothesis class of the parametric policies is large enough to represent all policies. 
Note that in  GRPO iterations the policy at iteration $n$ depends upon the policy $\pi_{n-1}$ via the probability of success $p_{n-1}$, as well on the reference policy via the $\mathsf{KL}$ regularizer. 

\section{GRPO Dynamics: Fixed Point iteration for Probability of Success  } \label{sec:Recursion}

Our goal in this Section is to analyze the dynamics of the GRPO iterations given in \eqref{eq:GRPO_iterations}.

\begin{theorem} [GRPO Policy Dynamics]
Optimal GRPO iterations policies solving \eqref{eq:GRPO_iterations} satisfy the following recursion, for $n\geq 1$:
\[ \pi_{n}(o|q)  = \frac{1}{Z_{n-1}(q)}\pi_{\mathrm{ref}}(o|q) \exp \left(\frac{1}{\beta} \left( \omega^+_{\varepsilon}(p_{n-1}(q)) \mathbbm{1}_{r(q,o)=1}  - \omega^-_{\varepsilon}(p_{n-1}(q))   \mathbbm{1}_{r(q,o)=0} \right) \right),  \]
where 
$ Z_{n-1}(q) = p_{\mathrm{ref}}(q) \exp \left(\frac{1}{\beta} \omega^+_{\varepsilon}(p_{n-1}(q))   \right) +  (1-p_{\mathrm{ref}}(q))  \exp \left(- \frac{1}{\beta} \omega^-_{\varepsilon} (p_{n-1} (q)) \right) ,$
where the weights $\omega^+_\varepsilon$ and $\omega^-_{\varepsilon}$ are given in \eqref{eq:cases}, the probability of success $p_{n-1}(q)$ of policy $\pi_{n-1}(\cdot | q)$  is given in   \eqref{eq:psuccess}, and $p_{\mathrm{ref}}(q)$ is the probability of success of the reference policy $\pi_{\mathrm{ref}} (\cdot|q)$:  $ p_{\mathrm{ref}}(q) = \mathbb{E}_{o \sim \pi_{\mathrm{ref}} (\cdot | q) } \mathbbm{1}_{r(q,o)=1}$.
\label{theo:piexpression}
\end{theorem}

We turn now to the recursion satisfied by the probability of success  $p_n(q)$ of the policy $\pi_{n} (\cdot | q)$, we have the following theorem that shows that this success probability satisfies a fixed point iteration:

\begin{theorem}[GRPO's Probability of Success Fixed Point Iteration] Assume $p_{\mathrm{ref}}>0$, define  for $\beta>0$:
$$h_{\varepsilon,p_{\mathrm{ref}}}(p)   = \frac{1}{1+ \frac{1-p_{\mathrm{ref}}}{p_{\mathrm{ref}}} \exp \left(-\frac{1}{\beta}  \frac{1}{\sqrt{p(1-p) + \varepsilon}}\right) }$$
The probability of success along GRPO's iteration satisfies the following fixed point iteration i.e we have almost surely  for all $q$ for $n\geq 1$
\begin{equation}
p_{n}(q) = h_{\varepsilon,p_{\mathrm{ref}}(q)}(p_{n-1} (q)) ,
\label{eq:fixedIteration}
\end{equation}
and $p_0(q) = p_{\mathrm{ref}} (q)$.
\label{theo:IterationProba}
\end{theorem}

\begin{remark}[Importance of $\varepsilon >0$] Note if $\varepsilon=0$, $h_{\varepsilon, p_{\mathrm{ref}}}$ is no longer continuous at $0$ and $1$ and we can no longer guarantee existence of fixed points on $[0,1]$.
\end{remark}


\begin{figure}[ht!]  
    \centering  
    \includegraphics[width=0.5\textwidth]{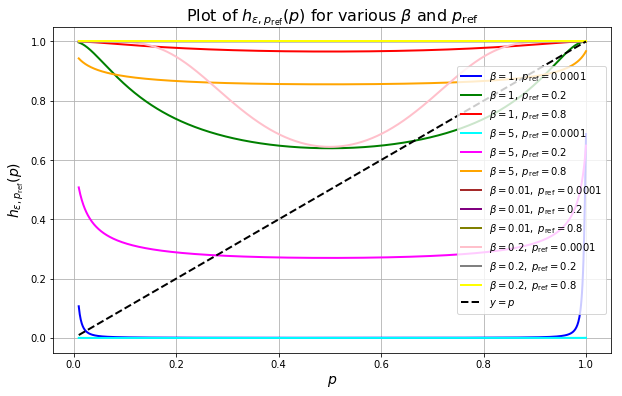}  
    \caption{Fixed points as function of $\beta$ and $p_{\mathrm{ref}}$ for $\varepsilon = 1e^{-5}$.}  
    \label{fig:hepsilon}  
\end{figure}

\noindent We study in the following proposition properties of the function $h_{\varepsilon, p_{\mathrm{ref}}}$:

\begin{proposition} [Properties of $h_{\varepsilon, p_{\mathrm{ref}}}$]  $h_{\varepsilon, p_{\mathrm{ref}}}$ satisfies the following properties:
\begin{itemize}
\item  Existence of fixed points: $h_{\varepsilon, p_{\mathrm{ref} }}$ is continuous $ from [0,1]$ to $[0,1]$ and hence admits at least a fixed point $p^*$ in $[0,1]$ (no guarantees for a unique fixed point)
\item Monotonicity: $h'_{\varepsilon,p_{\mathrm{ref}}}(p) 
 = - h_{\varepsilon, p_{\mathrm{ref}}}(p) (1-h_{\varepsilon, p_{\mathrm{ref}}}(p) ) \frac{1-2p}{2\beta\,[p(1-p)+\varepsilon]^{3/2}}$

\begin{itemize}
\item if $p<\frac{1}{2}$, $h'_{\varepsilon,p_{\mathrm{ref}}}(p)<0 $  and $h_{\varepsilon, p_{\mathrm{ref} }}(p)$ is decreasing 
\item if $p>\frac{1}{2}$ $h'_{\varepsilon,p_{\mathrm{ref}}}(p)>0 $  and $h_{\varepsilon, p_{\mathrm{ref}}}(p)$ is increasing 
\item if $p= \frac{1}{2}$ $h'_{\varepsilon, p_{\mathrm{ref}}}(p)= 0$ and $p=\frac{1}{2}$ achieves its minimum
\end{itemize}
\item Let $\logit(p) = \log\!\left(\frac{p}{1-p}\right),  
\sigma(x) = \frac{1}{1+e^{-x}},
((\logit\circ\sigma)(x)=x,\ (\sigma\circ\logit)(p)=p).$ Define $\Omega_{\varepsilon}(p)= \omega^{+}_{\varepsilon}(p) + \omega^{-}_{\varepsilon}(p)=  (p(1-p)+ \varepsilon)^{-\frac{1}{2}}.$
 We have: \[h_{\varepsilon,p_{\mathrm{ref}}}(p)=\sigma\!\left(\logit\big(p_{\mathrm{ref}}\big)
+\frac{\Omega_{\varepsilon}(p)}{\beta}\right).\]
\end{itemize}
 \label{pro:propreties}
\end{proposition}

We drop in the sequel $q$, when referring to the sequence $p_{n}(q)$, and write for short $p_n$. If the sequence defined in  GRPO's probability of success iteration \eqref{eq:fixedIteration} converges we have therefore by continuity of $h_{\varepsilon, p_{\mathrm{ref}}}$:
\[ p_{\infty} = \lim_{n\to \infty} p_{n} = \lim_{n\to \infty} h_{\varepsilon, p_{\mathrm{ref}}}(p_{n-1})=h_{\varepsilon,p_{\mathrm{ref}}} (\lim_{n\to \infty}  p_{n-1}) = h_{\varepsilon,p_{\mathrm{ref}}}(p_{\infty})  ,\]
and hence $p_{\infty} = h_{\varepsilon,p_{\mathrm{ref}}} (p_{\infty})$, and the limit point probability of success of GRPO $p_{\infty}=p^*$ is a fixed point of $h_{\varepsilon,p}$ (fixed points exist by virtue of proposition \ref{pro:propreties}). Note that the fixed point $p^*$ is indeed function of $q$, and this dependency  in $h_{\varepsilon,p_{\mathrm{ref}} }$ is via $p_{\mathrm{ref}}(q)$.

We see in Figure \ref{fig:hepsilon} various plots of the function \( h_{\varepsilon, p_{\mathrm{ref}}} \) for different values of \( \beta \) and initialization \( p_{\mathrm{ref}} \), as well as the plot of the function \( y = p \). Fixed points correspond to the intersections of this line with the curve of \( h_{\varepsilon, p_{\mathrm{ref}}} \). We see that the fixed points are not unique in general, and \( p^* = 1 \) is almost always a fixed point.


\subsection{GRPO: Fixed Point Iteration  and   Success Amplification } \label{sec:amplificationConvergence}


Note that from the third item in proposition \ref{pro:propreties}
the PoS recurrence in Theorem \ref{theo:IterationProba} can be written in terms of success odds as follows:
\[
\logit\big(p_n(q)\big)=\logit\big( p_{\mathrm{ref}}(q)\big)+\frac{\Omega_{\varepsilon}(p_{n-1}(q))}{\beta}
\]

\begin{theorem}[GRPO amplifies the probability of success] For  $q\sim \rho_{\mathcal{Q}} $ assume $0<p_{\mathrm{ref}}(q)<1$. Let $p^*(q)$ be a fixed point of $h_{\varepsilon,p_{\mathrm{ref}}(q)}$ we have   $p^*(q)>p_{\mathrm{ref}}(q)$.%
 \label{theo:probaAmplification}
\end{theorem}

We see from Theorem \ref{theo:probaAmplification} for any prompt $q$,  the fixed point PoS \( p^*(q) \) of the GRPO iteration leads to an amplification of the probability of success of the reference model \( p_{\mathrm{ref}}(q) \).  Note if $p_{\mathrm{ref}}(q)=0$ or $p_{\mathrm{ref}}(q)=1$ , the iteration will lead to $p^*(q)=0$ and $p^*(q)=1$ respectively.
In this case the fixed point is not necessarily stable and a condition on $\beta$ is needed for its stability (See appendix \ref{sec:refonly} )

\section{Mirror GRPO: Mirror Descent with GRPO Calibrated Reward}\label{Sec:MirrorGRPO}

\label{sec:mirror-only}

Note that we previously considered GRPO with no-clipping and with a $\mathsf{KL}$ regularization to $
\pi_{\mathrm{ref}}$.


We consider here a mirror GRPO with a regularization to $\pi_{n-1}$ in addition to $\pi_{\mathrm{ref}}$. For $n\geq 1$:
\begin{equation}
\label{eq:mirror-obj1}
\max_{\pi}\; \mathbb{E}_{q\sim \rho_{\mathcal{Q}}} \left(\mathbb{E}_{\pi(\cdot|q)}
A_{n-1}(q,\cdot)-\;\beta\, \left( \alpha \mathsf{KL}\!\Big(\pi(\cdot\mid q)\,\Big\|\,\pi_{\mathrm{ref}}(\cdot\mid q)\Big)+ (1-\alpha ) \mathsf{KL}\!\Big(\pi(\cdot\mid q)\,\Big\|\,\pi_{n-1}(\cdot\mid q)\Big)\right) \right),
\end{equation}
where:
\begin{equation}
\label{eq:group-const-adv}
A_{n-1}(q,o)=
\begin{cases}
+\omega^+_\varepsilon(p_{n-1}(q)), & r(q,o)=1\\[2pt]
-\omega^-_\varepsilon(p_{n-1}(q)), &r(q,o)=0,
\end{cases}
\end{equation}
and $p_{n-1}(q)=\mathbb{P}_{\pi_{n-1}(\cdot|q)}(r(q,o)=1)$, and $\pi_0=\pi_{\mathrm{ref}}$.

If $\alpha=1$ we obtain $\mathsf{KL}$ regularization to the $\pi_{\mathrm{ref}}$. If $\alpha=0$, we obtain mirror regularization to the previous iteration without considering the reference. Many  recent works suggested using $\alpha=0$ such as DAPO \citep{yu2025dapoopensourcellmreinforcement} i.e removing the regularization to the reference in GRPO while maintaining the clipping.  Note that proximal methods with regularization to previous iterates play the same role of clipping \citep{tomar2021mirrordescentpolicyoptimization,gunter2024appleintelligencefoundationlanguage}. Indeed PPO style clipping \citep{schulman2017proximal} was introduced as an approximation of such  proximal mirror descent.

We study in the following the case $\alpha=0$, the general case $\alpha>0$ is analyzed in Appendix \ref{sec:two-kl}.

Theorem \ref{theo:mirropolicy} gives the optimal policy for Mirror-GRPO iterations, and its corresponding PoS recurrence: 

\begin{theorem}[Mirror-GRPO, $\alpha=0$]\label{theo:mirropolicy}
\label{thm:mirror-only}
Fix $\alpha=0$ and a prompt $q$ and let $\beta>0$.
Let $\Omega_\varepsilon(p)= \frac{1}{\sqrt{p(1-p)+\varepsilon}}.$ Then the following holds:
\begin{enumerate}[leftmargin=1.1em, itemsep=4pt, topsep=4pt]

\item\emph{Optimal policy.}
The maximizer of \eqref{eq:mirror-obj1} is
\[ \pi_{n}(o|q)  = \frac{1}{Z_{n-1}(q)}\pi_{n-1}(o|q) \exp \left(\frac{1}{\beta} \left( \omega^+_{\varepsilon}(p_{n-1}(q)) \mathbbm{1}_{r(q,o)=1}  - \omega^-_{\varepsilon}(p_{n-1}(q))   \mathbbm{1}_{r(q,o)=0} \right) \right),  \]
where 
$ Z_{n-1}(q) = p_{n-1}(q) \exp \left(\frac{1}{\beta} \omega^+_{\varepsilon}(p_{n-1}(q))   \right) +  (1-p_{n-1}(q))  \exp \left(- \frac{1}{\beta} \omega^-_{\varepsilon} (p_{n-1} (q)) \right) .$

\item\emph{PoS and odds recurrences.} The PoS of $\pi_{n}(\cdot|q)$, $p_{n}(q)$, satisfies  the following recurrence: 
\begin{align}
\logit(p_n(q))
&=\logit(p_{n-1}(q))\;+\;\frac{\Omega_\varepsilon(p_{n-1}(q))}{\beta},\\
p_n(q)
&=h_{\varepsilon,\beta}(p_{n-1}(q))=\sigma\!\Big(\logit(p_{n-1}(q))+\Omega_\varepsilon(p_{n-1}(q))/\beta\Big).
\label{eq:posmirror}
\end{align}
\end{enumerate}
\end{theorem}

When compared with Theorem \ref{theo:IterationProba}, we see that $p_{n-1}(q)$, replaces $p_{\mathrm{ref}}(q)$ in the logit inside the sigmoid. 

\begin{theorem}[Monotone Improvement and Absorbing Fixed Points]
Fix a prompt $q$,the PoS iterations $p_n(q)$  of Mirror-GRPO ($\alpha=0$) have the following properties:

\begin{enumerate}[leftmargin=1.1em, itemsep=4pt, topsep=4pt]

\item\textbf{Monotone improvement and absence of interior fixed points.}
For any $p_{n-1}\in(0,1)$, $\Omega_\varepsilon(p_{n-1})/\beta>0$, hence $\logit(p_n)>\logit(p_{n-1})$ and $p_n>p_{n-1}$.
Consequently, the equation $p=\sigma(\logit(p)+\Omega_\varepsilon(p)/\beta)$ has \emph{no} solution in $(0,1)$.
The only fixed points are at the boundary: $p\in\{0,1\}$.

\item\textbf{Convergence and stability.} The fixed points of Mirror-GRPO iterations ($\alpha=0$) satisfy:  
\begin{enumerate}[label=(\alph*), itemsep=3pt]
\item If $p_{\mathrm{ref}}(q)\in(0,1)$, then $(p_n(q))_{n}$ is strictly increasing and bounded by $1$, hence $p_n\uparrow 1$.
\item  If $p_{\mathrm{ref}}(q)\in (0,1)$,  $p^*=1$ is (globally)  stable fixed point: $\lim_{n\to\infty}p_n(q)=1$.
\item If $p_{\mathrm{ref}}(q)=0$ then $p_n(q)= 0$ for all $n$.
\end{enumerate}
\end{enumerate}
\end{theorem}

When compared with Theorem \ref{theo:probaAmplification}, we see for non-zero $p_{\mathrm{ref}}(q)$, Mirror-GRPO iterations of probability of success converges to $1$ that is a stable fixed point, whereas for GRPO with only reference regularization we may have an interior fixed point $p^*(q)> p_{\mathrm{ref}}(q)$. In both case for zero $p_{\mathrm{ref}}(q)$, GRPO with reference regularization or Mirror GRPO don't create successes, and the fixed point success remains at zero. From a practical point of view removing the reference regularization is convenient as one does not need to keep in memory the reference model in addition to the current model . In addition it has more favorable PoS guarantees than reference regularization only. Nevertheless in many situations one wants to achieves good performance on a task via RL training while maintaining the knowledge of the reference model and hence the case $\alpha>0$ is also of interest, we study this case fully in  Appendix \ref{sec:two-kl}. The main takeaway in that scenario where interpolate between $\alpha=0$ and $\alpha=1$, is that we lose monotonic improvement.  The PoS iteration incurs what we call  a Rényi correction  that  encodes the mismatches in success and failures between the reference and the previous iteration and we are back to an interior fixed point in $(0,1)$ and no guarantees of global stability as in the mirror-GRPO case.  

\section{Dr. GRPO and mean-only Normalization}\label{sec:meanonly}

We turn now to another reward normalization proposed in Dr. GRPO \citep{liu2025understanding}. \citet{liu2025understanding} suggests to use a mean-only normalization in GRPO. In our notations this corresponds to the following reward calibration: 

\begin{equation}
\label{eq:cases}
A_{\pi_{\theta_{\text{old}}}}(q,o)=
\begin{cases}
+\omega^+(p(q)), & r(q,o)=1,\\
-\omega^-(p(q)), & r(q,o)=0,
\end{cases}
\quad
\omega^+(p)=1-p,
\qquad
\omega^-(p)=p.
\end{equation}

This results in the following (no clipping) Dr. GRPO iterations for PoS: 
\[ \logit(p_{n}(q)) = \logit (p_{\mathrm{ref}}(q)) + \frac{1}{\beta}.\]
and the following for Mirror Dr. GRPO
\[ \logit(p_{n}(q)) = \logit (p_{n-1}(q)) + \frac{1}{\beta}.\]

These expressions can be obtained applying Theorem \ref{theo:general} in Appendix \ref{sec:generalapp} for this particular weighting with $\Omega(p)=1$. 

When Compared with (no) clip GRPO, DR. GRPO has a trivial constant fixed point $p^*(q)=\sigma(\logit p_{\mathrm{ref}}(q)+\frac{1}{\beta})$. While for Mirror Dr. GRPO , $L_{n}(q)= \logit( p_{n}(q))$ is an arithmetic progression and $L_{n}(q)= L_{\mathrm{ref}}(q)+ \frac{n}{\beta}$ and $p_{n}(q)\uparrow 1$ for non degenerate $p_{\mathrm{ref}}(q)\in (0,1)$. Comparing to Mirror GRPO we have a similar convergence to a PoS of 1 but the iteration are adaptive in the case of Mirror GRPO: 
\[ \logit(p_{n}(q)) = \logit (p_{n-1}(q)) + \frac{\Omega_{\varepsilon}(p_{n-1}(q))}{\beta }= \logit (p_{n-1}(q)) + \frac{1}{\beta (\sigma^2_{n-1}(q)+ \varepsilon)^{\frac{1}{2}}},\]
we can think that the variance normalization corresponds to mean-only normalization with an adpative effective $\beta_{\text{eff}}= \beta \sqrt{\sigma^2_{n-1}(q)+ \varepsilon}$. For low variance we make large increments in the logits of PoS and for high variance, we make smaller increments in the logits of PoS.  

\section{Discussion and Conclusion}
Table \ref{tab:Summary} in the Appendix summarizes different flavors of  GRPO  we studied in this paper and their corresponding probability of success iterations. 

The main dimensions these variants differ on are:  1) the reward calibration: mean and variance normalization as in the original GRPO or mean-only normalization  as in Dr GRPO \citep{liu2025understanding}. Our theory showed that the normalization results in different weighting schemes, non linear in the PoS for GRPO and linear in the PoS for Dr GRPO.  2)  As discussed earlier the analysis of the PPO style  clipping to maintain the policy updates in the vicinity of the old policy is challenging and it has been shown to be more stable to use mirror policy descent to train LLMs with RL \citep{gunter2024appleintelligencefoundationlanguage}. Hence we distinguish GRPO variants also with respect to the anchor distribution on which the KL regularization is applied : no-clip refers to $\pi_{\mathrm{ref}}$ regularization only. Mirror corresponds to the KL regularization given in \eqref{eq:mirror-obj1} with respect to the previous iterate ($\alpha=0$), we also consider the regularization to both reference and previous iteration (two-KL) for $\alpha>0$.  For $\alpha=0$, we see that we obtain a monotonic improvement in the PoS whereas mixing the reference and the previous iterate in the iterations  does not guarantee monotonic improvement. The PoS iteration in this case depends on the mismatch in success and failures between the reference and the previous iteration that we quantify in Appendix \ref{sec:two-kl} via a Rényi correction.

From a practical point of view Table \ref{tab:Summary} suggests the following in using GRPO in training LLMs:

\begin{takeaway}[Practical Takeaways]
\begin{lessons}
 \item    \textbf{Normalization equivalence.} The mean{+}variance normalization in GRPO  is equivalent from PoS point of view  to mean-only normalization using an adaptive  KL regularization $\beta_{\mathrm{eff}}= \beta \sigma(q)$. One can use either a fixed $\beta$ and get constant increments in log PoS odds  via mean-only calibration, or use mean calibration with $\beta_{\mathrm{eff}}$ as a KL regularizer which results in adaptive increments that are equivalent to mean + variance normalization without having to divide by the variance in the advantage. 
\item \textbf{Mirror versus Clipping and Reference Mixing} Mirror GRPO (KL to previous iteration only ) instead of clipped GRPO  guarantees monotonic improvement  and convergence to PoS of 1 for non degenerate $p_{\mathrm{ref}} $. Mirror GRPO has  the best theoretical and practical guarantees. Adding the reference regularization to this mirror descent results in an internal fixed point and no monotonic improvement is guaranteed.  Practically speaking, keeping a reference policy in memory increases bandwidth/latency and can slow training for large models.  
 \item  \textbf{Coverage and exploration.}   In all cases GRPO does not create successes and 0 is an absorbing fixed point if $p_{\mathrm{ref}}(q)=0$. Hence it is important to maintain successes   exploration (e.g., temperature, entropy bonus, or data mixing) so successes have nonzero support. 
 \end{lessons}
\end{takeaway}

\appendix

\bibliographystyle{abbrvnat}
\bibliography{iclr2024_conference,rlhf} 

\newpage
\appendix
\section*{Summary}
\begin{table}[ht!]
\scriptsize
\setlength{\tabcolsep}{4pt}
\renewcommand\arraystretch{1.12}
\centering
\begin{tabularx}{\linewidth}{l X X X}
\toprule
Method & Weights ($w^+,\,w^-,\,\Omega$) & Recurrence PoS (logit) & Fixed Point (FP)/ Stability \\
\midrule

\makecell[l]{(No-clip)\\ GRPO\\ (mean-var )} &
\makecell[l]{
$w^+=(1-p_{n-1})\,\Omega_\varepsilon(p_{n-1})$\\
$w^-=p_{n-1}\,\Omega_\varepsilon(p_{n-1})$\\
$\Omega_{\varepsilon}(p) = (p(1-p)+ \varepsilon)^{-\frac{1}{2}} $
} &
\makecell[l]{
$\logit p_n=\logit p_{\mathrm{ref}}$\\
$\quad+\;\Omega_\varepsilon(p_{n-1})/\beta$
} &
\makecell[l]{
$p^\star=\sigma\!\big(\logit p_{\mathrm{ref}}+\Omega_\varepsilon(p^\star)/\beta\big)$ (implicit).\\
Converges if $h(p)=\sigma(\logit p_{\mathrm{ref}}+\Omega_\varepsilon(p)/\beta)$\\
is a contraction: $\sup_p|h'(p)|<1$.
} \\

\addlinespace[2pt]
\makecell[l]{(No-clip)\\ Dr.\ GRPO\\ (mean-only)} &
\makecell[l]{
$w^+=1-p_{n-1}$,\quad $w^-=p_{n-1}$\\
$\Omega=1$
} &
\makecell[l]{
$\logit p_n=\logit p_{\mathrm{ref}}$\\
$\quad+\;1/\beta$
} &
\makecell[l]{
$p^\star=\sigma(\logit p_{\mathrm{ref}}+1/\beta)$ (one step).\\
Trivially stable under re-application.
} \\

\addlinespace[2pt]
\makecell[l]{Mirror\\ GRPO\\ (mean-var )} &
\makecell[l]{
$w^+=(1-p_{n-1})\,\Omega_\varepsilon(p_{n-1})$\\
$w^-=p_{n-1}\,\Omega_\varepsilon(p_{n-1})$\\
$\Omega(p)=(p(1-p)+\varepsilon)^{-\frac{1}{2}}$
} &
\makecell[l]{
$\logit p_n=\logit p_{n-1}$\\
$\quad+\;\Omega_\varepsilon(p_{n-1})/\beta$
} &
\makecell[l]{
No interior FP; $p_n\uparrow 1$ (non-degenerate).\\
$p{=}1$ global; $p{=}0$ absorbing only\\
 if success support $=0$.
} \\

\addlinespace[2pt]
\makecell[l]{Mirror\\ Dr.\ GRPO\\ (mean-only)} &
\makecell[l]{
$w^+=1-p_{n-1}$,\quad $w^-=p_{n-1}$\\
$\Omega=1$
} &
\makecell[l]{
$ \logit p_n=\logit p_{n-1}$\\
$\quad+\;1/\beta$
} &
\makecell[l]{
No interior FP; $p_n\uparrow 1$.\\
$p{=}1$ global (non-degenerate starts).
} \\

\addlinespace[2pt]
\makecell[l]{Mirror GRPO\\ + $\pi_{\mathrm{ref}}$\\ (two-KL,\ mean+var)} &
\makecell[l]{
$w^+=(1-p_{n-1})\,\Omega_\varepsilon(p_{n-1})$\\
$w^-=p_{n-1}\,\Omega_\varepsilon(p_{n-1})$\\
$\Omega=\Omega_\varepsilon(p_{n-1})$
} &
\makecell[l]{
$\logit p_n=\alpha\,\logit p_{\mathrm{ref}}$\\
$\quad+(1-\alpha)\,\logit p_{n-1}$\\
$\quad+\;\Delta_R(q)+\Omega_\varepsilon(p_{n-1})/\beta$
} &
\makecell[l]{
$\logit p^\star=\logit p_{\mathrm{ref}}+\dfrac{\Delta_R^\star}{\alpha}$\\
$\quad+\;\dfrac{\Omega_\varepsilon(p^*)}{\alpha\beta}$ (if finite).\\
Affine contraction in log-odds if $\Delta_R$ bounded;\\
per-step monotonicity not guaranteed.
} \\

\addlinespace[2pt]
\makecell[l]{Mirror Dr.\ GRPO\\ + $\pi_{\mathrm{ref}}$\\ (two-KL,\ mean)} &
\makecell[l]{
$w^+=1-p_{n-1}$,\quad $w^-=p_{n-1}$\\
$\Omega=1$
} &
\makecell[l]{
$\logit p_n=\alpha\,\logit p_{\mathrm{ref}}$\\
$\quad+(1-\alpha)\,\logit p_{n-1}$\\
$\quad+\;\Delta_R(q)+1/\beta$
} &
\makecell[l]{
$\logit p^\star=\logit p_{\mathrm{ref}}+\dfrac{\Delta_R^\star}{\alpha}$\\
$\quad+\;\dfrac{1}{\alpha\beta}$ (unique FP).\\
Contraction in log-odds with rate $(1-\alpha)$;\\
$p^\star>p_{\mathrm{ref}}$ if $\Delta_R^\star+1/\beta>0$.
} \\

\bottomrule
\end{tabularx}
\caption{GRPO variants with fixed $\beta$ and mixed penalty $\beta\big[\alpha\,\mathrm{KL}(\pi\|\pi_{\mathrm{ref}})+(1-\alpha)\,\mathrm{KL}(\pi\|\pi_{n-1})\big]$.}
\label{tab:Summary}
\end{table}

\section{Experimental Validation}\label{Sec:Exp}

\noindent \textbf{Setup} We use  the \href{https://huggingface.co/datasets/openai/gsm8k}{\texttt{GSM8K}} dataset from \cite{Cobbe2021} (MIT license), and\\ {\texttt{Qwen/Qwen2.5-0.5B-Instruct}} (Apache 2.0 license) by \cite{Yang2024d} as the reference policy. We use GRPO implementation in   TRL \citep{vonwerra2022trl}, and train  on  the training split of \texttt{GSM8K} on a node with 8 GPUs (GPU0 for the vLLM server and 7 other GPUs for distributed training). We use a learning rate $5e^{-6}$, clipping $\varepsilon=0.2$  and the KL regularizer $\beta=0.1$, and $\mu$ in Algorithm \ref{alg:iter-grpo} is set to $\mu=10$. Other hyperparameters are given in Appendix \ref{app:assets} . We use the correctness of the LLM output as a reward.\\ 

\noindent \textbf{Success Rate Amplification} The success rate of the policy is then evaluated on the test set consisting of 1319 math questions, where for each question the success rate is evaluated using $50$ samples.   We see a success rate amplification from $\pi_{\mathrm{ref}}$ originally (averaged on all prompts) at $21\%$ to $37.5\%$ at the end of the GRPO epoch.\\

\noindent \textbf{Trajectory of Success rates Along GRPO Iterations} We randomly select few prompts from GSM8K test set and plot in Figure \ref{fig:traj} the trajectory of the success rate of the model along the GRPO iteration (estimated from 50 samples from the model for each prompt). The success rate is computed from checkpoints of the model along the GRPO training.  We see that the trajectory of the success rate $p(q)$ resembles the trajectory of a fixed point algorithm (see Figure \ref{fig:iterationConvergence} in Appendix  \ref{app:plots} ). For some points the convergence is fast to the limit point $p^*=1$, for others we see an oscillatory behavior (similar to the one in  last row in Figure \ref{fig:iterationConvergence}). Interestingly when $p_{\mathrm{ref}}=0$, the probability of success does not move much along GRPO iterations as predicted by our theory.

\begin{figure}[htbp]
  \centering
    \centering
    \includegraphics[width=0.5\linewidth]{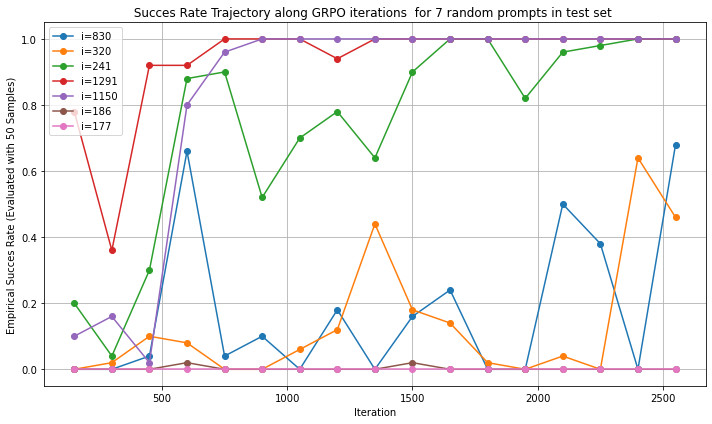}
    \caption{Success rate trajectory of the model on randomly selected prompts along the GRPO iters.}
    \label{fig:traj}
\end{figure}

\section{Algorithm}\label{sec:alg}
 \begin{algorithm}[!ht]
  \small
  \caption{Iterative GRPO with verifiable rewards, modified from  \citep{shao2024deepseekmath}}

  \begin{algorithmic}[1]
\State  \textbf{Input} initial policy model $\pi_{\theta_{\text{init}}}$; verifiable reward $r$; task prompts $\mathcal{D}$; 
  hyperparameters $\epsilon$, $\beta$, $\mu$
    \State policy model $\pi_\theta \leftarrow \pi_{\theta_{\text{init}}}$
      \For{$n = 1, \dots, M$}
      \State Sample a batch $\mathcal{D}_b$ from $\rho_{\mathcal{Q}}$
      \State Update the old policy model $\pi_{\theta_{\textrm{old}}} \leftarrow \pi_{\theta}$ 
      \State Sample $G$ outputs $\{o_i\}_{i=1}^G \sim \pi_{\theta_{\textrm{old}}} (\cdot \mid q) $ for each question $q \in \mathcal{D}_b$
      \State Compute rewards $\{r_i\}_{i=1}^{G}$ for each sampled output $o_i$ by running verifiable reward $r$ 
      \State Compute $\hat{A}(q,o_i)$ using  \eqref{eq:casesmeanvar}, where $\hat{p}(q) = \hat{p}_{\theta_{\text{old}}}(q) = \frac{1}{G} \sum_{i=1}^G \mathbbm{1}_{r(q,o_i)=1}$
      \For{GRPO iteration = 1, \dots, $\mu$}
        \State Update the policy model $\pi_{\theta}$ by maximizing GRPO objective with gradient ascent 
      \EndFor
    \EndFor 
    \State \noindent   \textbf{Output} $\pi_\theta$
  \end{algorithmic}
  \label{alg:iter-grpo}
\end{algorithm}

\newpage
\newpage

\section{(No-Clipping ) GRPO : Proofs of Section \ref{sec:Recursion}}

\begin{proof} [Proof of Theorem \ref{theo:piexpression}] 

The objective in  Equation \eqref{eq:GRPO_iterations} is concave and hence setting the first order optimality conditions (See for example \citep{mroueh2024informationtheoreticguaranteespolicy} ) we obtain:
\[ \pi_{n}(o|q)  = \frac{1}{Z_{n-1}(q)}\pi_{\mathrm{ref}}(o|q) \exp \left(\frac{1}{\beta} \left( \omega^+_{\varepsilon}(p_{n-1}(q)) \mathbbm{1}_{r(q,o)=1}  - \omega^-_{\varepsilon}(p_{n-1}(q))   \mathbbm{1}_{r(q,o)=0} \right) \right),  \]

where 
\begin{align*}
Z_{n-1}(q) &= \int d\pi_{\mathrm{ref}}(o|q) \exp \left(\frac{1}{\beta} \left( \omega^+_{\varepsilon}(p_{n-1}(q)) \mathbbm{1}_{r(q,o)=1}  - \omega^-_{\varepsilon}(p_{n-1}(q))   \mathbbm{1}_{r(q,o)=0} \right) \right) \\
& = \mathbb{E}_{o \sim \pi_{\mathrm{ref}}(\cdot| q) } \mathbbm{1}_{r(q,o)=1} \exp \left(\frac{1}{\beta} \left( \omega^+_{\varepsilon}(p_{n-1}(q)) \mathbbm{1}_{r(q,o)=1}  - \omega^-_{\varepsilon}(p_{n-1}(q))   \mathbbm{1}_{r(q,o)=0} \right) \right)\\
& + \mathbb{E}_{o \sim \pi_{\mathrm{ref}}(\cdot| q) } \mathbbm{1}_{r(q,o)=0}  \exp \left(\frac{1}{\beta} \left( \omega^+_{\varepsilon}(p_{n-1}(q)) \mathbbm{1}_{r(q,o)=1}  - \omega^-_{\varepsilon}(p_{n-1}(q))   \mathbbm{1}_{r(q,o)=0} \right) \right)\\
& =\exp \left(\frac{1}{\beta} \omega^+_{\varepsilon}(p_{n-1}(q))\right)  \mathbb{E}_{o \sim \pi_{\mathrm{ref}}(\cdot| q) } \mathbbm{1}_{r(q,o)=1} + \exp \left(-\frac{1}{\beta} \omega^-_{\varepsilon}(p_{n-1}(q))\right) \mathbb{E}_{o \sim \pi_{\mathrm{ref}}(\cdot| q) } \mathbbm{1}_{r(q,o)=0} \\
& = p_{\mathrm{ref}}(q)\exp \left(\frac{1}{\beta} \omega^+_{\varepsilon}(p_{n-1}(q))\right) + (1-p_{\mathrm{ref}}(q ))  \exp \left(-\frac{1}{\beta} \omega^-_{\varepsilon}(p_{n-1}(q))\right) ,
\end{align*}
where $$ p_{\mathrm{ref}}(q) =p_0(q) =\mathbb{E}_{o \sim \pi_{\mathrm{ref}} (\cdot | q) } \mathbbm{1}_{r(q,o)=1}.$$

\end{proof}

\begin{proof} [Proof of Theorem \ref{theo:IterationProba}] Replacing $\pi_n(\cdot|q)$ by its expression from Theorem \ref{theo:piexpression} we have:
\begin{align*}
p_{n}(q) & = \mathbb{E}_{ o\sim \pi_{n}(.|q)} \mathbbm{1}_{r(q,o)=1}\\
& = \frac{1}{Z_{n-1}(q)} \int d\pi_{\mathrm{ref}}(o|q) \exp \left(\frac{1}{\beta} \left( \omega^+_{\varepsilon}(p_{n-1}(q)) \mathbbm{1}_{r(q,o)=1}  - \omega^-_{\varepsilon}(p_{n-1}(q))   \mathbbm{1}_{r(q,o)=0} \right) \right) \mathbbm{1}_{r(q,o)=1}  \\
& = \frac{1}{Z_{n-1}(q)} \exp\left( \frac{1}{\beta} \omega^+_{\varepsilon}(p_{n-1} (q)) \right) \mathbb{E}_{\pi_{\mathrm{ref}}} \mathbbm{1}_{r(q,o)=1}\\
& = \frac{p_{\mathrm{ref}} (q) \exp\left( \frac{1}{\beta} \omega^+_{\varepsilon}(p_{n-1} (q)) \right)}{Z_{n-1}(q)}\\
& = \frac{p_{\mathrm{ref}} (q) \exp\left( \frac{1}{\beta} \omega^+_{\varepsilon}(p_{n-1} (q)) \right)}{ p_{\mathrm{ref}}(q)\exp \left(\frac{1}{\beta} \omega^+_{\varepsilon}(p_{n-1}(q))\right) + (1-p_{\mathrm{ref}}(q ))  \exp \left(-\frac{1}{\beta} \omega^-_{\varepsilon}(p_{n-1}(q))\right) }
\end{align*}
Replacing the weights expressions from \eqref{eq:cases} we obtain: 
\begin{equation}
p_{n}(q) = \frac{p_{\mathrm{ref}}\exp\left (\frac{1}{\beta} \left( \frac{1-p_{n-1} (q)}{\sqrt{p_{n-1}(q) (1-p_{n-1}(q)) + \varepsilon}} \right )\right) }{ p_{\mathrm{ref}}\exp \frac{1}{\beta} \left( \frac{1-p_{n-1}(1)}{\sqrt{p_{n-1}(q) (1-p_{n-1}(q)) + \varepsilon}} \right)   +  (1-p_{\mathrm{ref}})  \exp \frac{1}{\beta} \left( - \frac{p_{n-1}(q)}{\sqrt{p_{n-1}(q)(1-p_{n-1}(q)) + \varepsilon}} \right)  }
\end{equation}

Define \[ h_{\varepsilon,p_{\mathrm{ref}}}(p) = \frac{p_{\mathrm{ref}} \exp\left (\frac{1}{\beta} \left( \frac{1-p}{\sqrt{p (1-p) + \varepsilon}} \right )\right) }{ p_{\mathrm{ref}}\exp \frac{1}{\beta} \left( \frac{1-p}{\sqrt{p (1-p) + \varepsilon}} \right)   +  (1-p_{\mathrm{ref}})  \exp \frac{1}{\beta} \left( - \frac{p}{\sqrt{p(1-p) + \varepsilon}} \right)  }
\]

\noindent We see therefore that GRPO's probability of success satisfies the following iteration : 

\[ p_{n} (q) = h_{\varepsilon,p_{\mathrm{ref}}}(p_{n-1}(q)) . \]

\noindent We assume here that $0<p_{\mathrm{ref}}<1$.We can simplify $h_{\varepsilon} (p)$ as follows:

\begin{align*}
h_{\varepsilon,p_{\mathrm{ref}}}(p)  &= \frac{1}{1+ \frac{1-p_{\mathrm{ref}}}{p_{\mathrm{ref}}} \exp \frac{1}{\beta} \left(  \frac{-p }{\sqrt{p(1-p) + \varepsilon}} -  \frac{1-p }{\sqrt{p(1-p) + \varepsilon}} \right)  }\\
& = \frac{1}{1+ \frac{1-p_{\mathrm{ref}}}{p_{\mathrm{ref}}} \exp \left(-\frac{1}{\beta}  \frac{1}{\sqrt{p(1-p) + \varepsilon}}\right) }.
\end{align*}
\end{proof}

\begin{proof}[Proof of Proposition \ref{pro:propreties}]

\noindent \textbf{Existence of fixed points} For $\varepsilon >0$ $h_{\varepsilon,p_{\mathrm{ref}}}$ is continuous function from \([0,1]\) to \([0,1]\)  and hence by Brouwer's Fixed Point Theorem at least a fixed point $p^*$ exists in $[0,1]$, i.e $\exists p^* \in [0,1]$ such that $p^* = h_{\varepsilon, p_{\mathrm{ref}}} (p^*)$.\\

\noindent \textbf{Monotonicity} Let $\sigma(z) = \frac{1}{1+ \exp(-z)}$ and let $A = \frac{1-p_{\mathrm{ref}}}{p_{\mathrm{ref}}}  $ and $B(p) = \frac{1}{\beta}  \frac{1}{\sqrt{p(1-p) + \varepsilon}}$ hence we have:

\[ h_{\varepsilon, p_{\mathrm{ref}}} (p) = \sigma\left( z(p) \right)\]

\noindent where $$z(p) = - \log(A) + B(p)  $$
we have $$z'(p) =   B'(p) = -\frac{1-2p}{2\beta\,[p(1-p)+\varepsilon]^{3/2}}  $$

\noindent Let us compute the derivative : 
\begin{align*}
h'_{\varepsilon,p_{\mathrm{ref}}}(p) & = \sigma(z(p)) (1-\sigma(z(p))) z'(p)\\
& = - \sigma(z(p)) (1-\sigma(z(p)))  \frac{1-2p}{2\beta\,[p(1-p)+\varepsilon]^{3/2}}
\end{align*}

\begin{itemize}
\item if $p<\frac{1}{2}$, $h'_{\varepsilon,p_{\mathrm{ref}}}(p)<0 $  and $h_{\varepsilon,p_{\mathrm{ref}}}$ is decreasing 
\item if $p>\frac{1}{2}$ $h'_{\varepsilon,p_{\mathrm{ref}}}(p)>0 $  and $h_{\varepsilon,p_{\mathrm{ref}}}$ is increasing 

\item if $p= \frac{1}{2}$ $h'_{\varepsilon,p_{\mathrm{ref}}}(p)= 0 $ 
\end{itemize}
Turning to third point:
\begin{align*}
 h_{\varepsilon, p_{\mathrm{ref}}} (p) &= \sigma\left( z(p) \right)\\
 &= \sigma \left( \log \frac{p_{\mathrm{ref}}}{1-p_{\mathrm{ref} }} + \frac{1}{\beta}  \frac{1}{\sqrt{p(1-p) + \varepsilon}} \right)\\
 & = \sigma\left(\logit (p_{\mathrm{ref}})  + \frac{1}{\beta}  \frac{1}{\sqrt{p(1-p) + \varepsilon}}\right)
 \end{align*}

and hence:
\begin{equation}
    \logit \left( h_{\varepsilon, p_{\mathrm{ref}}} (p)\right) =\logit (p_{\mathrm{ref}})  + \frac{1}{\beta}  \frac{1}{\sqrt{p(1-p) + \varepsilon}}. 
\end{equation}

\end{proof}

\subsection{Proofs of Section \ref{sec:amplificationConvergence} }

\begin{proof}[Proof of Theorem \ref{theo:probaAmplification}]
We claim that any fixed point \(p^*\) of \(h_\varepsilon\) satisfies
\[
p^*>p_{\mathrm{ref}}.
\]

We have for all $\beta,\varepsilon >0$ $\exp\!\left(-\frac{1}{\beta}\frac{1}{\sqrt{p_{\mathrm{ref}}(1-p_{\mathrm{ref}})+\varepsilon}}\right) <1$.
\begin{align*}
h_{\varepsilon,p_{\mathrm{ref}}}(p)-p_{\mathrm{ref}}&=\frac{1}{1+\frac{1-p_{\mathrm{ref}}}{p_{\mathrm{ref}}}\,\exp\!\left(-\frac{1}{\beta}\frac{1}{\sqrt{p(1-p)+\varepsilon}}\right)}- p_{\mathrm{ref}}\\
&>\frac{1}{1+\frac{1-p_{\mathrm{ref}}}{p_{\mathrm{ref}}}} -p_{\mathrm{ref}}\\
& = p_{\mathrm{ref}}- p_{\mathrm{ref}}\\
& =0.
\end{align*}
Hence for any fixed point we have $h_{\varepsilon,p_{\mathrm{ref}}}(p^*)=p^*$ and we have $p^*>p_{\mathrm{ref}}.$
\end{proof}
\subsection{Stability fixed Point GRPO with Reference Only Regularization } \label{sec:refonly}
We drop in the sequel the dependency on $q$ to simplify notations and turn to the second question regarding the convergence of the GRPO sequence given in \eqref{eq:fixedIteration} to a fixed point \( p^* \) of \( h_{\varepsilon, p_{\mathrm{ref}}} \). Given the properties of \( h_{\varepsilon, p_{\mathrm{ref}}} \), we can characterize the limit point of the GRPO iteration as \( n \to \infty \) as follows, as a consequence of the local Banach fixed-point theorem:

\begin{theorem}[Local Fixed Point Convergence] Let $p^*$ be a fixed point of $h_{\varepsilon,p_{\mathrm{ref}}}$and assume that have $|h'_{\varepsilon,p_{\mathrm{ref}}}(p^*)|<1$.Given that $h_{\varepsilon,p_{\mathrm{ref}}}$ and $h'_{\varepsilon,p_{\mathrm{ref}}}$ are continuous in $[0,1]$, then there exists $\delta>0$ such the iteration $p_{n} = h_{\varepsilon,p_{\mathrm{ref}}}(p_{n-1})$ converges to $p^*$, if $p_0 =p_{\mathrm{ref}} \in [p^*-\delta, p^*+\delta]$. In other words under this condition we have: 
$$\lim _{n\to \infty} p_{n}  = p^*.$$
\label{theo:convergence}
\end{theorem}

\begin{lemma} Let $p^*$ be a fixed point: $p^*= h_{\varepsilon,p_{\mathrm{ref}}} (p^*)$, then we have:
\begin{align*}
h'_{\varepsilon,p_{\mathrm{ref}}}(p^*)&= - h_{\varepsilon, p_{\mathrm{ref}}}(p^*) (1-h_{\varepsilon, p_{\mathrm{ref}}}(p^*) ) \frac{1-2p^*}{2\beta\,[p^*(1-p^*)+\varepsilon]^{3/2}}\\
&=  p^* (1-p^*) \frac{2p^*-1}{2\beta\,[p^*(1-p^*)+\varepsilon]^{3/2}}\
\end{align*}
\end{lemma}

One condition for local convergence is therefore to have: 
$|h'_{\varepsilon,p_{\mathrm{ref}}}(p^*)| =  p^* (1-p^*) \frac{| 2p^*-1|}{2\beta\,[p^*(1-p^*)+\varepsilon]^{3/2}} <1 $
which is satisfied if : 
$\beta > \mathcal{B}(p^*)=p^* (1-p^*) \frac{| 2p^*-1|}{2[p^*(1-p^*)+\varepsilon]^{3/2}}.$

We see from Figure \ref{fig:lowerboundconvergence} in Appendix \ref{app:plots} the lower bound on \( \beta \) required to ensure local convergence of GRPO iterations to a fixed point \( p^* \). Figure \ref{fig:iterationConvergence} in Appendix  \ref{app:plots} shows iteration \eqref{eq:fixedIteration} as a function of \( n \) for different values of \( \beta \) and \( p_{\mathrm{ref}} \). We see that in most cases, there is a sharp transition where we observe fast convergence to \( 1 \) or to a fixed point \( p^* \). For \( \beta = 5 \) and \( p_{\mathrm{ref}} = 0.001 \), we see a divergent behavior.

\begin{remark} Note that the condition on $\beta$ is stated conditionally on a prompt $q$, to obtain results uniformly on $q$ we need to take $\sup$ on $q$ in all lower bounds.   
\label{rem:qdependency}
\end{remark}

\paragraph{Practical Implications.}  In practical implementations GRPO is applied successively in stages where $\pi_{\mathrm{ref}}$ is set to the last iteration from the GRPO training in each stage \citep{shao2024deepseekmath}. In light of our theory this ensures that we are amplifying the probability of success w.r.t the new $\pi_{\mathrm{ref}}$, coming the previous GRPO stage.

\begin{proof}[Proof of Theorem \ref{theo:convergence}] This is a direct application of local Banach fixed point theorem:

\begin{theorem}[Local Contraction Mapping for One-Dimensional Functions]
Let \( f : \mathbb{R} \to \mathbb{R} \) be continuously differentiable, and suppose that \( x^* \in \mathbb{R} \) is a fixed point of \( f \) (i.e., \( f(x^*) = x^* \)). Assume that \( f' \) is continuous and that
\[
\lvert f'(x^*) \rvert < 1.
\]
Then, by the continuity of \( f' \), there exists a radius \( r > 0 \) and a constant \( k \) with
\[
\lvert f'(x) \rvert \le k < 1 \quad \text{for all } x \in [x^*-r, \, x^*+r].
\]
Consequently, \( f \) is a contraction on the interval \( I = [x^* - r, x^* + r] \), and for any initial guess \( x_0 \in I \), the iteration defined by
\[
x_{n+1} = f(x_n)
\]
converges to the unique fixed point \( x^* \) in \( I \).
\end{theorem}

\end{proof}

\section{Mirror GRPO: Proof of Section \ref{Sec:MirrorGRPO} }\label{sec:generalapp}
\begin{theorem}[General Theorem with general weights and anchor policy]
\label{theo:general}
\[ \pi^*= \mathcal{P}(\nu, \pi_{\circ})=\argmax_{\pi} \mathbb{E}_{\pi(\cdot|q)} A_{\nu}(\cdot,q) - \beta \mathsf{KL} (\pi|| \pi_{\circ})\]
where 
\begin{equation}
\label{eq:grpo-weights}
A_{\nu}(q,o)=
\begin{cases}
+\omega^+(p_\nu), & r(q,o)=1,\\
-\omega^-(p_\nu), & r(q,o)=0,
\end{cases}
\end{equation}
where $p_{\nu}= \mathbb{P}_{\nu(\cdot|q)}(r(q,\cdot) =1) $. Let $\Omega(p)=\omega^+(p) + \omega^-(p). $
The following holds:
\begin{enumerate}
\item \[ \pi^* (o|q)=\frac{ \pi_{\circ} (o|q) \exp A_{\nu}(q,o)}{ p_{\pi_{\circ}}(q) \exp(\omega^+(p_\nu(q))) + (1- p_{\pi_{\circ}}(q)) \exp(- \omega^-(p_\nu(q)))}   \]
\item  Let $\pi^{\circ}_{n-1}(\cdot|q), n \geq 1$ a sequence of anchor probabilities, and $p^\circ_{n-1}(q)$ their corresponding PoS. Let $p_{n}=p_{\pi_{n}}$ where 
$\pi_{n}$  defined as follows : 

 \[ \pi_{n}(q) =  \mathcal{P}( \pi_{n-1}(q),  \pi^{\circ}_{n-1}(q) ), \]
we have: 

\[ \logit (p_{n} (q) ) = \logit \Big( p^{\circ}_{n-1} (q) \Big)+ \Omega (p_{n-1}(q) )  \]

and \[ p_{n}(q) = \sigma \Big( \logit \Big( p^{\circ}_{n-1} (q) \Big)+ \Omega (p_{n-1}(q) ) .\Big)\]

\end{enumerate} 
\end{theorem}

\begin{proof}
The proof of item 1 is the same as in Theorem \ref{theo:piexpression}. Turning to the second point we have by taking expectation on success events:
\begin{align*}
p_*(q) &= \frac{ p_{\pi_{\circ}}(q) \exp(w^+ (p_{\nu}(q)) )}{p_{\pi_{\circ}}(q) \exp(\omega^+(p_\nu(q))) + (1- p_{\pi_{\circ}}(q)) \exp(- \omega^-(p_\nu(q)))}\\
&= \frac{1}{1+ \exp (- \logit (p_{\pi_{\circ}}(q) ) - \omega^+(p_{\nu}(q)) - \omega^-(p_{\nu}(q)) )}\\
& = \sigmoid (\logit (p_{\pi_{\circ}}(q) ) + \Omega(p_{\nu}(q)))
\end{align*}
and hence using that sigmoid and logit are inverse we have:
\[ \logit(p_*) = \logit (p_{\pi_{\circ}}(q) ) + \Omega(p_{\nu}(q))\]

\end{proof}

\begin{proof}[Proof of Theorem \ref{theo:mirropolicy}] The theorem is immediate applying Theorem \ref{theo:general} with anchors $\pi_{n-1}$.

\end{proof}

\begin{proof}[Proof of Theorem \ref{theo:mirropolicy}]
\emph{(1) Monotonicity and no interior fixed points.} Let $L_{n} = \logit(p_{n})$.
For $p\in(0,1)$, $\Omega_\varepsilon(p)=1/\sqrt{p(1-p)+\varepsilon}>0$, so \eqref{eq:posmirror} implies $L_n>L_{n-1}$ and hence $p_n>p_{n-1}$. An interior fixed point would solve $L=L+\Omega_\varepsilon(p)/\beta$, impossible since the increment is strictly positive.

It is easy to verify that $p =0$ and $p=1$ are fixed points :
\[h_{\varepsilon,\beta}(0) = \sigmoid ( \logit(0) + \frac{1}{\beta \sqrt{\varepsilon}} ) = \sigmoid(-\infty)=0\]
\[ h_{\varepsilon,\beta}(1) = \sigmoid ( \logit(1) + \frac{1}{\beta \sqrt{\varepsilon}} ) = \sigmoid(+\infty)=1 \]

\emph{(2) Convergence and stability.}
(1) If $p_0= p_{\mathrm{ref}}\in(0,1)$, then  $(p_n)$ is strictly increasing and bounded by $1$, so $p_n\uparrow \bar p\le 1$,  and the limit point is $\bar{p}=1$ the fixed point. 
(2) the fixed point is unique and stable if $p_{\mathrm{ref}} \in(0,1)$.
(3) If $p_{\mathrm{ref}}=0$, $p_1=h_{\varepsilon,\beta}(0) = 0 $, and so on, zero is an absorbing fixed point.

\end{proof}

\section{GRPO with Two KL Regularizers: PoS Recursion, and Fixed-Point }
\label{sec:two-kl}

Consider the following iteration 
\begin{equation}
\label{eq:mirror-obj2KL}
\pi_{n}=\argmax_{\pi}\; \mathbb{E}_{q\sim \rho_{\mathcal{Q}}} \left(\mathbb{E}_{\pi(\cdot|q)}
 A_{n-1}(q,\cdot)-\;\beta\, \left( \alpha \mathsf{KL}\!\Big(\pi(\cdot\mid q)\,\Big\|\,\pi_{\mathrm{ref}}(\cdot\mid q)\Big)+ (1-\alpha ) \mathsf{KL}\!\Big(\pi(\cdot\mid q)\,\Big\|\,\pi_{n-1}(\cdot\mid q)\Big)\right) \right),
\end{equation}

\begin{lemma}[Geometric Mean]
\label{lem:two-kl-collapse}
For any distributions $\pi,\pi_{\mathrm{ref}},\pi^{\circ}$ let $\alpha>0$

\[
\alpha \mathsf{KL}(\pi\|\pi_{\mathrm{ref}})+ (1-\alpha)\mathsf{KL}(\pi\|\pi^{\circ})
= \mathsf{KL}(\pi\|\bar\pi^{(\alpha)}) + C(\pi_{\mathrm{ref}},\pi^{\circ}),
\]
where $\bar\pi^{(\alpha)}\propto \pi_{\mathrm{ref}}^{\alpha}\pi^{\circ (1-\alpha)}$ and $C$ is constant in $\pi$.
\label{ref:geometricmean}
\end{lemma}

\begin{proof}
See for example \citep{aminian2025theoreticalanalysisklregularizedrlhf}.
\end{proof}

By Lemma \ref{ref:geometricmean}, we can rewrite GRPO objective with two KL regularization to previous iteration and to the reference as a single KL regularizer to their geometric mean as follows:

 \begin{equation}
\label{eq:mirror-obj}
\pi_{n}=\argmax_{\pi}\; \mathbb{E}_{q\sim \rho_{\mathcal{Q}}} \left(\mathbb{E}_{\pi(\cdot|q)}
 A_{n-1}(q,\cdot)-\;\beta\, \mathsf{KL} (\pi || \tilde{\pi}^{(\alpha)}_{n-1}) \right),
\end{equation}
where 
\[ \tilde{\pi}^{(\alpha)}_{n-1} \propto \pi_{\mathrm{ref}}^{\alpha}\pi^{ (1-\alpha)}_{n-1}\]

To apply Theorem \ref{theo:general} we need to have an expression of the PoS under  the anchor $\tilde{\pi}^{(\alpha)}_{n-1}$, as function of $p_{\mathrm{ref}}$ and $p_{n-1}$ so we obtain a recurrence in $p_n$. 

Define the following success and failure conditional probabilities:
 \[
p_{\mathrm{ref},S}(o|q):=\frac{\pi_{\mathrm{ref}}(o\mid q)\,\mathbf 1_{\{r(q,o)=1\}}}{p_{\mathrm{ref}}(q)},\quad
p_{n-1,S}(o|q):=\frac{\pi_{n-1}(o\mid q)\,\mathbf 1_{\{r(q,o)=0\}}}{p_{n-1}(q)},\qquad \]
and 
\[
p_{\mathrm{ref},F}(o|q):=\frac{\pi_{\mathrm{ref}}(o\mid q)\,\mathbf 1{_\{r(q,o)=0\}}}{1-p_{\mathrm{ref}}(q)},\quad
p_{n-1,F}(o|q):=\frac{\pi_{n-1}(o\mid q)\,\mathbf 1_{\{r(q,o=0)\}}}{1-p_{n-1}(q)}.
\]

and let \[D_{\alpha} (P || Q) = \frac{1}{\alpha-1} \log \int p^\alpha q^{(1-\alpha)},\]

be the Rényi divergence of order $\alpha in (0,1)$.
\begin{lemma} [PoS geometric mean] The probability of success of the geometric mean $\tilde{\pi}^{(\alpha)}_{n-1}$ satisfies: 

\[ \logit  \tilde{p}^{(\alpha)}_{n-1} = \alpha \logit(p_{\mathrm{ref}}(q)) +(1-\alpha) \logit (p_{n-1}(q)) + ( \alpha-1 ) \left(  D_{\alpha}( p_{\mathrm{ref},S}|| p_{n-1,S})  - D_{\alpha}( p_{\mathrm{ref},F}|| p_{n-1,F}  )\right)    \]
\label{lem:posgeom}
\end{lemma}

\begin{proof}
Let $w_{S}= \int \mathbbm{1}_{r(q,o)=1} \tilde{\pi}^{(\alpha)}_{n-1}$ and $w_{F}= \int \mathbbm{1}_{r(q,o)=0} \tilde{\pi}^{(\alpha)}_{n-1}$.

\begin{align}
 \tilde{p}^{(\alpha)}_{n-1} & = \frac{\int \mathbbm{1}_{r(q,o)=1} \tilde{\pi}^{(\alpha)}_{n-1}}{ \int \mathbbm{1}_{r(q,o)=1} \tilde{\pi}^{(\alpha)}_{n-1} +  \int \mathbbm{1}_{r(q,o)=0} \tilde{\pi}^{(\alpha)}_{n-1} }\\
 & = \frac{1}{1+ \frac{w_{F}}{w_{S}}}
 \end{align}
 
 \begin{equation}
 \frac{w_{F}}{w_{S}}= \frac{\int \mathbbm{1}_{r(q,o)=0} \pi^{\alpha}_{\mathrm{ref}}\pi_{n-1}^{1-\alpha} }{\int \mathbbm{1}_{r(q,o)=1} \pi^{\alpha}_{\mathrm{ref}}\pi_{n-1}^{1-\alpha} } . 
 \end{equation}

It is easy to see that : 

\begin{equation}
 \frac{w_{F}}{w_{S}} = \frac{(1-p_{\mathrm{ref}(q)}^{\alpha}) (1-p_{n-1}(q))^{(1-\alpha)}\int p^{\alpha}_{\mathrm{ref}, F} (o|q)p_{n-1,F}^{1-\alpha}(o|q)}{(p_{\mathrm{ref}(q)}^{\alpha}) (p_{n-1}(q))^{(1-\alpha)}\int p^{\alpha}_{\mathrm{ref},S} (o|q)p_{n-1,S}^{1-\alpha}(o|q)}
\end{equation}

Taking log on both sides we have: 
\begin{align*}
\log \frac{w_{F}}{w_{S}} &= \log  \frac{(1-p_{\mathrm{ref}(q)})^{\alpha}}{p_{\mathrm{ref}}^{\alpha}(q)} + \log \left( \frac{(1-p_{n-1}(q))^{(1-\alpha)}}{p_{n-1}^{1-\alpha}(q)}\right) + \log  \int p^{\alpha}_{\mathrm{ref},F} (o|q)p_{n-1,F}^{1-\alpha}(o|q) - \int p^{\alpha}_{\mathrm{ref},S} (o|q)p_{n-1,S}^{1-\alpha}(o|q)\\
&=- \alpha \logit(p_{\mathrm{ref}}(q)) -(1-\alpha) \logit (p_{n-1}(q)) + (\alpha-1) \left( D_{\alpha}( p_{\mathrm{ref},F}|| p_{n-1,F}  ) - D_{\alpha}( p_{\mathrm{ref},S}|| p_{n-1,S}  )  \right),
\end{align*}
where 
\[D_{\alpha} (P || Q) = \frac{1}{\alpha-1} \log \int p^\alpha q^{(1-\alpha)},\]
is the Rényi divergence.

Finally we obtain:

\begin{align*}
 \tilde{p}^{(\alpha)}_{n-1} &= \frac{1}{1+ \exp(- \alpha \logit(p_{\mathrm{ref}}(q)) -(1-\alpha) \logit (p_{n-1}(q)) - ( \alpha-1 ) \left(  D_{\alpha}( p_{\mathrm{ref},S}|| p_{n-1,S})  - D_{\alpha}( p_{\mathrm{ref},F}|| p_{n-1,F}  )    \right))}\\
 &= \sigma( \alpha \logit(p_{\mathrm{ref}}(q)) +(1-\alpha) \logit (p_{n-1}(q)) + ( \alpha-1 ) \left(  D_{\alpha}( p_{\mathrm{ref},S}|| p_{n-1,S})  - D_{\alpha}( p_{\mathrm{ref},F}|| p_{n-1,F}  )    \right))
\end{align*}

This gives us finally: 

\[ \logit  \tilde{p}^{(\alpha)}_{n-1} = \alpha \logit(p_{\mathrm{ref}}(q)) +(1-\alpha) \logit (p_{n-1}(q)) + ( \alpha-1 ) \left(  D_{\alpha}( p_{\mathrm{ref},S}|| p_{n-1,S})  - D_{\alpha}( p_{\mathrm{ref},F}|| p_{n-1,F}  )    \right). \]

\end{proof}

\begin{theorem}[PoS recurrence for 2 KL regularizers] Fix $\alpha \in (0,1), \beta>0$.The probability of success for the iteration of GRPO with 2 KL regularizer given in \eqref{eq:mirror-obj2KL} satisfies the following recurrence: 
\begin{eqnarray}
\logit p_{n}(q) &= \alpha \logit(p_{\mathrm{ref}}(q)) +(1-\alpha) \logit (p_{n-1}(q)) + \underbrace{ ( 1-\alpha ) \left(  D_{\alpha}( p_{\mathrm{ref},F}|| p_{n-1,F})  - D_{\alpha}( p_{\mathrm{ref},S}|| p_{n-1,S}  )    \right)}_{\Delta_{R} \text{Rényi Correction}} \nonumber\\
&+ \frac{\Omega_{\varepsilon}(p_{n-1})(q)}{\beta} .
\label{eq:2KLRenyi}
\end{eqnarray}

\end{theorem}
\begin{proof}
The proof is direct consequence of theorem \ref{theo:general} with geometric mean anchor (as showed in Lemma \ref{lem:two-kl-collapse}). We replace  in theorem \ref{theo:general} the anchor PoS by its expression computed in lemma \ref{lem:posgeom}.
\end{proof}

Let $L_{n}(q)=\logit p_{n}(q) $ and $L_{\mathrm{ref}}(q)=\logit(p_{\mathrm{ref}}(q))  $, hence we have the following recursion:

 \[ L_{n}(q)- L_{\mathrm{ref}}(q) = (1-\alpha) (L_{n-1}(q) -L_{\mathrm{ref}}(q)) + (1-\alpha )(  D_{\alpha}( p_{\mathrm{ref},F}|| p_{n-1,F})  - D_{\alpha}( p_{\mathrm{ref},S}|| p_{n-1,S}  )   + \Omega_{\varepsilon}(p_{n}) \]
 
 Let us assume that : 
 
\[ D_{\alpha}( p_{\mathrm{ref},S}|| p_{n-1,S}  ) \leq  D_{\alpha}( p_{\mathrm{ref},F}|| p_{n-1,F}),\]

i.e conditional successes between reference and previous policy are closer than the failures than we have since $\Omega_{\varepsilon}>0$: 

\[ L_{n}(q)- L_{\mathrm{ref}}(q) \geq (1-\alpha) (L_{n-1}(q) -L_{\mathrm{ref}} (q)) \geq (1-\alpha)^n (L_0-L_{\mathrm{ref}})=0 \]
and we obtain that we amplify probability w.r.t to $p_{\mathrm{ref}}$.

\section{Back to Parametric GRPO Iterations}\label{sec:approx}

Let $\tilde{\pi}_n = \pi_{\theta_n}$, the sequence of parametric policies solutions of problem \eqref{eq:grpoparametric} produced by gradient descent for example as in Algorithm \ref{alg:iter-grpo}.  We make the following assumption on the total variation distance $\mathrm{TV}$ between  these parametric policies and the non-parametric GRPO policies $\pi_{n}$ given in Theorem \ref{theo:piexpression}. We show in this Section if we have approximate policies we can have still asymptotic convergence.

\begin{assumption} We assume $\tilde{\pi}_0= \pi_0 = \pi_{\mathrm{ref}}$ and assume for all $n\geq 1$, there exists $\delta_{n}\geq 0$ such that:
\[\mathrm{TV} (\tilde{\pi}_n  || \pi_n ) \leq \mathrm{TV} (\tilde{\pi}_{n-1}  || \pi_{n-1} ) + \delta_n,\]
such  that there exists $\delta^* \in [0,1)$ such that $\sum_{i=1}^n \delta_i \to \delta^* $ as $n \to \infty$. 
\label{ass:conv}
\end{assumption}

We have the following theorem:

\begin{theorem} Under Assumption \ref{ass:conv} and assuming that $p_{n}$ converges to $p^*$ the fixed point of  $h_{\varepsilon, p_{\mathrm{ref}}}$. Let $\tilde{p}_n$ the probability of success of the policy $\tilde{\pi}$ we have: 
\[\lim_{n\to \infty} | \tilde{p}_n -p^*|\leq 2 \delta^*.\]
In the case $\delta^*=0$, we have convergence to the fixed point.
\label{theo:convparametric}
\end{theorem}

In Assumption \ref{ass:conv} $\delta_n$ represent statistical, approximation and optimization errors. We see from Theorem \ref{theo:convparametric}, that  as long these error remain small, the probability of success of GRPO  parametric policy (estimated from samples and optimized for instance with gradient descent)  remains close to the fixed point probability success $p^*$.

\begin{proof}[Proof of Theorem \ref{theo:convparametric}]

Note that \[ \mathrm{TV} (\tilde{\pi}  || \pi) = \frac{1}{2} \sup_{||f||_{\infty}}  \mathbb{E}_{\tilde{\pi}} f -  \mathbb{E}_{{\pi}} f   \]
We have: 
\begin{align*}
|\tilde{p}_n - p_n | & = \Big | \mathbb{E}_{\tilde{\pi}_n} \mathbbm{1}_{r(q,o)=1} -  \mathbb{E}_{\pi_n} \mathbbm{1}_{r(q,o)=1} \Big|\\
&\leq 2 \mathrm{TV}(\tilde{\pi}_n || \pi_{n})\\
& \leq  2 \sum_{i=1}^n \delta_i + \mathrm{TV}(\tilde{\pi}_0, \pi_0)\\
&=  2 \sum_{i=1}^n \delta_i .
\end{align*}

Assume the sequence $p_n$ converges to $p^*$ the fixed point of $h_{\varepsilon, p_{\mathrm{ref}}}$.
Under Assumption \ref{ass:conv} we have :
\[ \lim_{n\to \infty} |\tilde{p}_n - p_n | \leq  2 \lim_{n\to \infty} \sum_{i=1}^n \delta_i =2\delta^*\]
\end{proof}

\section{Plots} \label{app:plots}
\begin{figure}[ht!]  
    \centering  
    \includegraphics[width=0.6\textwidth]{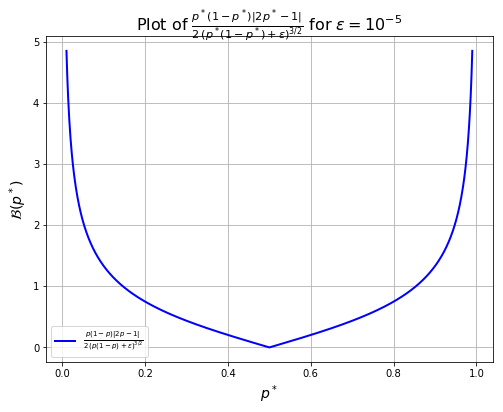}  
    \caption{Lower bound on $\beta$ to ensure local convergence of GRPO fixed point iteration.}  
    \label{fig:lowerboundconvergence}  
\end{figure}

\begin{figure}[!ht]  
    \centering  
    \includegraphics[width=0.85\textwidth]{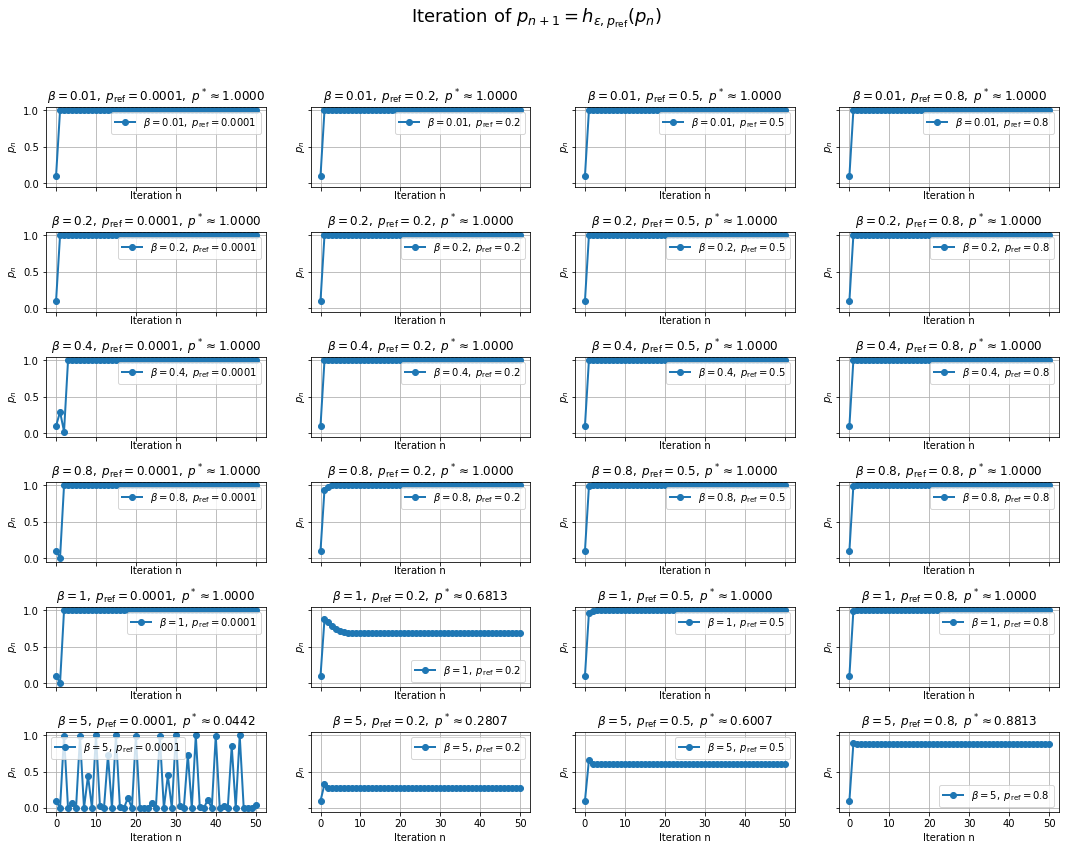}  
    \caption{ GRPO Recursion and convergence to fixed points of $h_{\varepsilon}$, for $\varepsilon = 1e^{-5}$}  
    \label{fig:iterationConvergence}  
\end{figure}

\section{Assets}\label{app:assets}

\paragraph{Hardware setup} All our experiments were run on one compute node with Dual 48-core Intel Xeon 8468, 2TB of RAM, 8 NVIDIA HGX H100 80GB SMX5, 8x 3.4TB Enterprise NVMe U.2 Gen4, and 10x NVIDIA Mellanox Infiniband Single port NDR adapters, running RedHat Enterprise Linux 9.5

\paragraph{GRPO Config Setup} We use  the group size $G=16$ and per-device batch size $16$ meaning each on each GPU a single prompt $x$ with $16$ corresponding responses is processed. To increase the overall batchsize we use gradient accumulation of $4$, ending with an effective batch size of prompts of $28$. 
The context length used for this experiment is $200$, and the sampling temperature is  set  to $\tau=0.1$.

\paragraph{Libraries} Our experiments rely on the open-source libraries \href{https://pytorch.org/}{\texttt{pytorch}} \citep{Paszke2019} (license: BSD), \href{https://github.com/huggingface/transformers}{\texttt{HuggingFace Transformers}} \citep{Wolf2020} (Apache 2.0 license), and \href{https://github.com/huggingface/trl}{\texttt{HuggingFace TRL}} \citep{Vonwerra2022} (Apache 2.0 license).
\end{document}